\documentclass{new_tlp}
\usepackage{mathptmx}
\pdfoutput=1
\newcommand{\AC}{\ensuremath{\mathcal{AC}}}

\title[ASP(\AC)]{ASP(\AC): Answer Set Programming with\\ Algebraic Constraints}

\author[T. Eiter and R. Kiesel]{THOMAS EITER and RAFAEL KIESEL\\
         Technical University Vienna, Vienna, Austria\\
         \email{\{thomas.eiter,rafael.kiesel\}@tuwien.ac.at}}

\usepackage{graphicx}

\usepackage[shortlabels]{enumitem}

\usepackage{url}

\let\proof\relax

\usepackage{amsmath}
\usepackage{amssymb}
\usepackage{amsthm}
\usepackage{times}
\usepackage{bbm}
\usepackage{nicefrac}
\usepackage[noabbrev]{cleveref} 
\crefname{equation}{}{}
\usepackage{multirow}
\usepackage{bm}

\usepackage{stmaryrd} 
\usepackage{latexsym} 

\newtheorem{theorem}{Theorem}

\newtheorem{define}[theorem]{Definition} 
 
\newtheorem{proposition}[theorem]{Proposition}
\newtheorem{example}{Example}

\DeclareMathOperator{\supp}{supp}

\newcommand{\rem}[1]{\ensuremath{{\rm \texttt{#1}}}}

\usepackage{listings}

\renewcommand{\overline}[1]{\vec{#1}}

\newcommand{\Sig}{\sigma}
\newcommand{\sorting}{r}

\newcommand{\splus}{\ensuremath{\bm{+}}}
\newcommand{\stimes}{\ensuremath{*}}
\newcommand{\bplus}{\ensuremath{\textstyle\Sigma}}
\newcommand{\btimes}{\ensuremath{\textstyle\Pi}}
\newcommand{\invplus}{\ensuremath{\bm{-}}}
\newcommand{\invtimes}[1]{\ensuremath{#1^{\bm{-1}}}}

\newcommand{\srzero}{\ensuremath{e_{\oplus}}}
\newcommand{\srone}{\ensuremath{e_{\otimes}}}
\newcommand{\srsplus}{\ensuremath{{\oplus}}}
\newcommand{\srstimes}{\ensuremath{{\otimes}}}
\newcommand{\srbplus}{\ensuremath{{\textstyle\bigoplus}}}
\newcommand{\srbtimes}{\ensuremath{{\textstyle\bigotimes}}}

\usepackage{xspace}
\usepackage{color,soul}
\usepackage{verbatim}

\usepackage[final]{pdfpages}
\usepackage{booktabs}

\newenvironment{myitemize}
{\begin{list}{$\bullet$}{%
\setlength{\topsep}{2pt} 
\setlength{\leftmargin}{0pt} 
\setlength{\itemindent}{10pt}}} 
{\end{list}}

\newcounter{myenumctr}

\makeatletter
\newcommand{\equationnumberhere}{%
\global\advance\c@equation\@ne\def\@currentlabel{\p@equation\theequation}%
(\@currentlabel)}
\makeatother

\usepackage{comment}

\newcommand{\myciteyear}[1]{\nocite{#1}\citeyear{#1}}

\pagerange{\pageref{firstpage}--\pageref{lastpage}}
\doi{S1471068401001193}

\begin{document}
\maketitle

\begin{abstract}
    Weighted Logic is a powerful tool for the specification of
   calculations over semirings that depend on qualitative
    information. Using a novel combination of Weighted Logic and
    Here-and-There (HT) Logic, in which this dependence is based on
    intuitionistic grounds, we introduce Answer Set Programming with
    Algebraic Constraints (ASP(\AC)), where rules may contain constraints
    that compare semiring values to weighted formula
    evaluations. Such constraints provide streamlined access to a manifold
    of constructs available in ASP, like aggregates, choice
    constraints, and arithmetic operators. They extend some of them and
    provide a generic framework for defining programs with algebraic computation, which can be fruitfully used e.g.\ for provenance semantics of datalog programs. While undecidable in general, expressive
    fragments of ASP(\AC)\;can be exploited for effective problem solving in
    a rich framework. This work is under consideration for acceptance in Theory and Practice of Logic Programming.
\end{abstract}

\allowdisplaybreaks

\section{Introduction}
Answer Set Programming (ASP) is a well-known non-monotonic declarative programming para\-digm. Due to the need for more expressiveness and succinct descriptions, it has been extended with many different constructs, ranging from nested expressions \cite{lifschitz1999nested} to weight constraints with conditionals \cite{niemela1999stable} and aggregates \cite{ferraris2011logic}. A more recent trend combines ASP with Constraint Processing (CP) employing both solvers for ASP and Satisfaction Modulo Theories (SMT), cf.\ \cite{lierler2014relating,DBLP:journals/ki/Janhunen18}. Many of these approaches keep nonmonotonicity on the ASP side, but also its use on the CP side was explored \cite{DBLP:journals/tplp/AzizCS13}.
Cabalar et al.~\myciteyear{cabalar2020asp,cabalar2020uniform} recently introduced a general non-monotonic integration of CP into ASP providing aggregates and conditionals for the specification of numeric values that depend on the satisfaction of formulas. The conditionals can be 
flexibly evaluated under the \emph{vicious circle}\/ ($vc$) or the \emph{definedness} principles ($df$). Their approach treats constraints as black boxes, leaving their syntax open,
and incorporates many previously introduced constructs.

An important feature of constraint extensions is the possibility to express (in)equations involving computations on an algebraic structure, whose solutions are accessible by the ASP rules. Basic such structures are  \emph{semirings} $\mathcal{R} = (R, \srsplus, \srstimes, \srzero, \srone)$, where $\srsplus$ and $\srstimes$ are addition and multiplication with neutral elements $\srzero$ and $\srone$, respectively. They have been considered for constraint semantics e.g.\  in \cite{bistarelli1997semiring} and were used  to provide rule languages with parameterised calculation in a uniform syntax but flexible semantics, cf.\ \cite{kimmig2011algebraic,eiter2020wlars}.


Notably, \emph{Weighted Logic}~\cite{droste2007weighted} links semirings with predicate logic, where \emph{weighted formulas}\/ are interpreted as algebraic expressions over a semiring $\mathcal{R}$. Similar to conditionals in \cite{cabalar2020uniform}, the value of a weighted formula $\alpha$ depends on the truth of the atoms in $\alpha$. E.g.\ the weighted formula $1 \splus \rem{deadline}\stimes(2 \splus \rem{pagelimit}\stimes 3)$ over the natural numbers $\mathbb{N}$ may represent how many cups of coffee a researcher drinks: if there is a $\rem{deadline}$ but no $\rem{pagelimit}$, the result is $1 + 1\cdot(2 + 0\cdot3) = 3$.

In recent work \cite{eiter2020wlars}, we exploited Weighted Logic in quantitative stream reasoning to assign weights to answer streams and aggregate over them. 
Here, we introduce a non-monotonic version of it in terms of \emph{First-Order Weighted Here-and-There Logic (FO-WHT)}. The resulting logic complements Cabalar et al.'s work on abstract constraints in several respects:
\begin{itemize}[align=left,leftmargin=*,topsep=0pt]
    \item it offers an elegant way of specifying calculations, aggregates and non-monotonic conditionals natively and without the need for auxiliary definitions;
    \item the semantics provides a natural alternative to the vicious circle and definedness principles which arguably combines their strengths;
    \item the parameterisation with semirings allows for terms in a uniform syntax that are not bound to the reals but can be over any semiring (which may be fixed at runtime).
\end{itemize}
\smallskip
\noindent
As customary for ASP we restrict ourselves to a fragment of FO-WHT Logic and introduce \AC-programs that allow for \emph{algebraic constraints}, i.e. constraints on the values of weighted formulas, in both heads and bodies of rules.
\AC-programs incorporate and extend many previous ASP constructs and thus provide a rich framework for declarative problem solving in a succinct form.
The main contributions of this paper are briefly summarised as follows:
\begin{itemize}[align=left, leftmargin=*,topsep=0pt]
    \item We introduce First-Order Weighed HT Logic and \AC-programs that include constraints over weighted formulas  (Section \ref{sec:sema}). By using a variant of HT Logic with non-disjoint sorts, such that variables can range over subsets shared by a domain and semirings, we enable the usage of constraints over different semirings within the same program.
    \item We consider different constructs in extensions of ASP like aggregates, choice constraints and conditionals, and we demonstrate how to model them in our framework. Further, we present the novel \emph{minimised constraints} that allow for subset minimal guessing (Section \ref{sec:comp}).
    \item To demonstrate the power of ASP(\AC)\;, we illustrate how provenance semantics for positive datalog programs can be elegantly encoded
    (Section \ref{sec:prov}).
    \item We consider different language aspects leading, firstly, to a broad class of \emph{safe} \AC-programs, which we show to be domain independent; and, secondly, to a characterisation of strong equivalence for \AC-programs by equivalence in FO-WHT Logic (Section \ref{sec:lang}).
    \item We obtain that in the propositional (ground) case, the complexity of disjunctive logic programs is retained, i.e.\ model checking (MC) and strong equivalence are co-NP-complete while answer set existence (SAT) is $\Sigma_2^{p}$-complete, if the used semirings satisfy a practically mild encoding condition. For safe non-ground programs, MC is feasible in EXPTIME at most; SAT and SE are undecidable in general, but expressive decidable fragments are available (Section~\ref{sec:complexity}).
\end{itemize}
\section{Preliminaries}
\label{sec:prel}
We start by introducing classical programs and their semantics. We use a variant of first-order HT semantics \cite{10.1007/978-3-540-89982-2_46} as this facilitates the generalisation of the semantics later on and is useful for work on strong equivalence. The variant is that we assign variables non-disjoint sorts, which lets us quantify over subsets of the domain. This slightly differs from other approaches in logic programming that use sorts, like \cite{balai2013sparc}, where the arguments of predicates are sorted.

We consider sorted first-order formulas over a \emph{signature} $\Sig = \langle \mathcal{D}, \mathcal{P}, \mathcal{X}, \mathcal{S}, \sorting \rangle$, where $\mathcal{D}$ is a set of domain elements, $\mathcal{P}$ a set of predicates, $\mathcal{X}$ a set of sorted variables, $\mathcal{S}$ a set of sorts and $\sorting : \mathcal{S} \rightarrow 2^{\mathcal{D}}$ a range function assigning each sort a subset of the domain. When $x \in \mathcal{X}$, we write $s(x)$ for the sort of $x$. Given a signature $\Sig$, we define the syntax of \emph{$\Sig$-formulas} by
\begin{equation}
\label{eq:sigma-formula}
\phi ::= \bot \mid p(\overline{x}) \mid \phi \rightarrow \phi \mid \phi \vee \phi \mid \phi \wedge \phi \mid \exists y \phi \mid \forall y \phi,
\end{equation}
where $p \in \mathcal{P}$, $\overline{x} = x_1, \dots, x_n$, with $x_i \in \mathcal{D}$ or $x_i \in \mathcal{X}$ and $y \in \mathcal{X}$; $p(\overline{x})$ is called a $\Sig$-\emph{atom}.
We define $\neg \phi = \phi \rightarrow \bot.$
A $\Sig$-\emph{sentence} is a $\Sig$-formula without free variables.
\begin{define}[HT Semantics]
\label{def:unweighted}
Let $\Sig = \langle \mathcal{D}, \mathcal{P}, \mathcal{X}, \mathcal{S}, \sorting \rangle$ be a signature and $\mathcal{I}^{H}, \mathcal{I}^{T}$ be $\Sig$-\emph{interpretations}, i.e. sets of $\Sig$-atoms without free variables over the predicates in $\mathcal{P}$ and elements in $\mathcal{D}$, s.t. $\mathcal{I}^{H} \subseteq \mathcal{I}^{T}$. Then $\mathcal{I} = (\mathcal{I}^{H}, \mathcal{I}^{T})$ is a $\Sig$-\emph{HT-interpretation} and $\mathcal{I}_{w} = (\mathcal{I}^{H}, \mathcal{I}^{T}, w)$, for $w \in \{H,T\}$, is a \emph{pointed} $\Sig$-\emph{HT-interpretation}.

Satisfaction of a $\Sig$-sentence $\phi$ w.r.t. a pointed $\Sig$-HT-interpretation $\mathcal{I}_{w} = (\mathcal{I}^{H}, \mathcal{I}^{T}, w)$ is defined as follows, 
where we have the reflexive order $\geq$ on $\{H,T\}$, with $T \geq H$:
 \begin{align*}
   \phantom{asdfasdf} \mathcal{I}_{w} &\not\models_{\Sig} \bot \\
    \mathcal{I}_{w} &\models_{\Sig} p(\overline{x}) &  &\iff & &p(\overline{x}) \in \mathcal{I}^{w}\\
    \mathcal{I}_{w} &\models_{\Sig} \phi \rightarrow \psi &  &\iff & &\mathcal{I}_{w'} \!\!\not\models_{\Sig}\phi \text{ or } \mathcal{I}_{w'} \models_{\Sig}\psi \text{ for all $w' \geq w$}\phantom{asdfasdf}\\
    \mathcal{I}_{w} &\models_{\Sig} \phi \vee \psi & &\iff & &\mathcal{I}_{w} \models_{\Sig} \phi \text{ or } \mathcal{I}_{w}\models_{\Sig} \psi\\
    \mathcal{I}_{w} &\models_{\Sig} \phi \wedge \psi & &\iff & &\mathcal{I}_{w} \models_{\Sig} \phi \text{ and } \mathcal{I}_{w}\models_{\Sig} \psi\\
    \mathcal{I}_{w} &\models_{\Sig} \exists x \phi(x) & &\iff & & \mathcal{I}_{w} \models_{\Sig} \phi(\xi) \text{, for some } \xi \in \sorting(s(x)) \\
    \mathcal{I}_{w} &\models_{\Sig} \forall x \phi(x) & &\iff & & \mathcal{I}_{w} \models_{\Sig} \phi(\xi) \text{, for all } \xi \in \sorting(s(x))
\end{align*}
When $T$ is a set of $\Sig$-sentences, then $\mathcal{I}_w \models_{\Sig} T$ if $\forall \phi \in T: \mathcal{I}_w \models_{\Sig} \phi$.  
\end{define}
The semantics of classical rules and programs is introduced as an instantiation of the above semantics for restricted signatures. Let $\Sig = \langle \mathcal{D}, \mathcal{P}, \mathcal{X}, \mathcal{S}, \sorting \rangle$ be a \emph{classical signature}, i.e. $\mathcal{S} = \{\top\}, \sorting(\top) = \mathcal{D}$. Then a \emph{rule} is of the form\smallskip

\centerline{$
r = H(r) \leftarrow B(r) = \phi \leftarrow \psi_1, \dots, \psi_n, \neg \theta_1, \dots, \neg \theta_m, 
$}
\smallskip

\noindent
where $\phi, \psi_i, \theta_j$ are $\Sig$-atoms, with free variables $x_1, \dots, x_k \in \mathcal{X}$. Its semantics is that of the $\Sig$-formula
$
\forall x_1, \dots, x_k \; B_{\wedge}(r) \rightarrow \phi \text{ where } B_{\wedge}(r) \text{ is } \psi_1\wedge \dots\wedge \psi_n\wedge \neg \theta_1\wedge \dots\wedge \neg \theta_m. 
$
Similarly, a \emph{program} $\Pi$ is a set of rules. 
\begin{define}[Equilibrium Model]
\label{def:eq-model}
Given a signature $\Sig$, a $\Sig$-interpretation $\mathcal{I}$ is an \emph{equilibrium model} of a (set of) $\Sig$-sentence(s) $\phi$ if $(\mathcal{I}, \mathcal{I}, H) \models_{\Sig} \phi$ and for all $\mathcal{I}' \subsetneq \mathcal{I}: (\mathcal{I}', \mathcal{I}, H) \not\models_{\Sig} \phi$.
\end{define}

To introduce weighted formulas, we first recall semirings.
\begin{define}[Semiring]
A \emph{semiring} $\mathcal{R} = (R, \srsplus, \srstimes, \srzero, \srone)$ is a set $R$ equipped with two binary operations $\srsplus$ and $\srstimes$, which are called \emph{addition} and \emph{multiplication}, such that
\begin{myitemize}
    \itemsep=0pt
    \item $(R, \srsplus)$ is a commutative monoid with identity element $\srzero$,
    \item $(R, \srstimes)$ is a monoid with identity element $\srone$,
    \item multiplication left and right distributes over addition,
    \item and multiplication by $\srzero$ annihilates $R$, i.e. $\forall r \in R: r\srstimes \srzero = \srzero = \srzero \srstimes r$.
\end{myitemize} 
\end{define}
\begin{example}[Semirings]
Some well known semirings are
\begin{myitemize}
\label{items:semirings}
    \itemsep=0pt
    \item $\mathbb{S} = (\mathbb{S}, +, \cdot, 0, 1), \text{ for } \mathbb{S} \in \{\mathbb{N,Z,Q,R}\}$, the semiring over the numbers in $\mathbb{S}$. It is typically used for arithmetic.
    \item $\mathbb{N}_{\infty} = (\mathbb{N}\cup\{\infty\}, +, \cdot, 0, 1)$, where $\infty + n = \infty$ and $\infty \cdot m = \infty$, for $m \neq 0$. It is typically used in the context of provenance.
    \item $\mathcal{R}_{\rem{max}} = (\mathbb{Q}\cup \{-\infty, \infty\}, \max, +, -\infty, 0)$, the max tropical semiring. It is typically used in the context of provenance and optimisation. 
    \item $\mathbb{B} = (\{0,1\}, \vee, \wedge, 0, 1)$, the Boolean semiring. It is typically used for classical Boolean constraints.
    \item $2^{A} = (2^{A}, \cup, \cap, \emptyset, A)$, the powerset semiring over the set $A$. It is typically used in the context of measure theory but also for set arithmetic and succinct specifications of multiple requirements.
\end{myitemize}
\end{example}
The connective $\srsplus$ (resp. $\srstimes$) in a semiring $\mathcal{R}$ is \emph{invertible}, if for every $r \in R$ (resp.\ $r \in R\setminus\{\srzero\}$) some $-r \in R$ (resp.\ $r^{-1} \in R$) exists s.t.\ $r \srsplus -r = \srzero$ (resp. $r\srstimes r^{-1} = \srone$); inverse elements are unique allowing us to use $-(\cdot), (\cdot)^{-1}$ as unary connectives. An \emph{ordered semiring}\/ is pair $(\mathcal{R}, >)$, where $\mathcal{R}$ is a semiring and $>$ is a strict total order on $R$.

For space reasons, we must omit introducing Weighted Logic \cite{droste2007weighted} explicitly and confine to compare our logic to the one by Droste and Gastin below.

\section{ASP(\AC)}
\label{sec:sema}
We start by introducing First-Order Weighted HT Logic. Intuitively it generalises First-Order HT Logic by replacing disjunctive connectives ($\vee, \exists$) by additive ones ($\splus, \bplus$), conjunctive ones ($\wedge, \forall$) by multiplicative ones ($\stimes, \btimes$), and accordingly, the neutral elements $\bot, \top$ by zero ($\srzero$) and one ($\srone$).
\begin{define}[Syntax]
For a signature $\Sig = \langle \mathcal{D}, \mathcal{P}, \mathcal{X}, \mathcal{S}, \sorting \rangle$, the \emph{weighted} $\Sig$-formulas over the semiring $\mathcal{R} = (R, \srsplus, \srstimes, \srzero, \srone)$ are of the form
\[
    \alpha ::= k \mid x \mid \phi \mid \alpha \rightarrow_{\mathcal{R}} \alpha \mid \alpha \splus \alpha \mid \alpha \stimes \alpha \mid \invplus \alpha \mid \invtimes{\alpha} \mid \bplus y \alpha \mid \btimes y \alpha,
\]
where $k \in R$, $x,y \in \mathcal{X}$ s.t. $\sorting(s(x)) \subseteq R$ (i.e., $x$ takes only values from $R$) and $\phi$ is a $\Sig$-formula. The use of $\invplus$ and $\invtimes{}$ require that $\srsplus$ and $\srstimes$ are invertible, the use of $\btimes y$ requires that $\srstimes$ is commutative. We define 
$\neg_{\mathcal{R}} \alpha = \alpha \rightarrow_{\mathcal{R}} \srzero$. A \emph{weighted} $\Sig$-sentence is a variable-free weighted $\Sig$-formula.
\end{define}
\begin{example}
Let $\Sig = \langle \mathbb{Q}, \{p\}, \{X\}, \{S\}, \{S \mapsto \mathbb{Q}\}\rangle$ and $s(X) = S$; thus, $X$ ranges over the rational numbers. Then $\bplus X p(X)\,{\stimes}\,X$ is a weighted $\Sig$-sentence over the semirings $\mathcal{R}_{\max},\mathbb{Q}$
but not over $\mathbb{N}$.
\end{example}
\begin{define}[Semantics]
Let $\Sig=\langle \mathcal{D}, \mathcal{P}, \mathcal{X}, \mathcal{S}, \sorting \rangle$ be a signature. The semantics of a weighted $\Sig$-sentence over semiring $\mathcal{R}$ w.r.t.\ $\mathcal{I}_{w} = (\mathcal{I}^{H}, \mathcal{I}^{T}, w)$ is 
inductively defined in Figure~\ref{fig:wSem}.

\begin{figure}
    \centering
\vspace{-5pt}
\[
\begin{gathered}
\begin{aligned}
    \\[-8pt]
    \llbracket k \rrbracket^{\Sig}_{\mathcal{R}}(\mathcal{I}_{w}) &= k, \text{ for $k \in R$}\\[8pt]
    \llbracket \invplus \alpha \rrbracket^{\Sig}_{\mathcal{R}}(\mathcal{I}_{w}) &= -(\llbracket \alpha \rrbracket^{\Sig}_{\mathcal{R}}(\mathcal{I}_{w}))\\
    \llbracket \invtimes{\alpha} \rrbracket^{\Sig}_{\mathcal{R}}(\mathcal{I}_{w}) &= (\llbracket \alpha \rrbracket^{\Sig}_{\mathcal{R}}(\mathcal{I}_{w}))^{-1}
\end{aligned}\hspace{5pt}
\begin{aligned}
    \llbracket \phi \rrbracket^{\Sig}_{\mathcal{R}}(\mathcal{I}_{w}) &= \left\{\begin{array}{ll}
        \srone, & \text{if } \mathcal{I}_{w} \models_{\Sig} \phi,\\
        \srzero, & \text{otherwise.}
        \end{array}\right. \text{, for $\Sig$-formulas $\phi$}\\
    \llbracket \alpha \splus \beta \rrbracket^{\Sig}_{\mathcal{R}}(\mathcal{I}_{w}) &= \llbracket \alpha \rrbracket^{\Sig}_{\mathcal{R}}(\mathcal{I}_{w}) \srsplus \llbracket \beta \rrbracket^{\Sig}_{\mathcal{R}}(\mathcal{I}_{w})\\
    \llbracket \alpha \stimes \beta \rrbracket^{\Sig}_{\mathcal{R}}(\mathcal{I}_{w}) &= \llbracket \alpha \rrbracket^{\Sig}_{\mathcal{R}}(\mathcal{I}_{w}) \srstimes \llbracket \beta \rrbracket^{\Sig}_{\mathcal{R}}(\mathcal{I}_{w})
\end{aligned}\\[2pt]
\hspace{3pt}\begin{aligned}[c]
    \llbracket \alpha \rightarrow_{\mathcal{R}} \beta \rrbracket^{\Sig}_{\mathcal{R}}(\mathcal{I}_{w}) &= \left\{\begin{array}{ll}
        \srone, & \text{if } \llbracket\alpha\rrbracket^{\Sig}_{\mathcal{R}} (\mathcal{I}_{w'}) = \srzero \text{ or } \llbracket\beta\rrbracket^{\Sig}_{\mathcal{R}} (\mathcal{I}_{w'}) \neq \srzero~\text{ for all }w' \geq w,\\
        \srzero, & otherwise.
        \end{array}\right.\\
    \llbracket \bplus x \alpha(x) \rrbracket^{\Sig}_{\mathcal{R}} (\mathcal{I}_{w}) &= \left\{\begin{array}{cl}
        \srbplus_{\xi \in \supp_{\srsplus}(\alpha(x), \mathcal{I}_{w})} \llbracket \alpha(\xi)\rrbracket^{\Sig}_{\mathcal{R}} (\mathcal{I}_{w}), & \text{if } 
        \supp_{\srsplus}(\alpha(x), \mathcal{I}_{w}) \textrm{ is finite},\\
        \text{undefined}, & \text{otherwise.}
    \end{array}\right.\\
    \llbracket \btimes x \alpha(x) \rrbracket^{\Sig}_{\mathcal{R}} (\mathcal{I}_{w}) &= \left\{\begin{array}{cl}
        \srbtimes_{\xi \in \supp_{\srstimes}(\alpha(x), \mathcal{I}_{w})} \llbracket \alpha(\xi)\rrbracket^{\Sig}_{\mathcal{R}} (\mathcal{I}_{w}), & \text{if } 
        \supp_{\srstimes}(\alpha(x), \mathcal{I}_{w}) \textrm{ is finite},\\
        \srzero, & \text{if } 
        \sorting(s(x)) \setminus \supp_{\srsplus}(\alpha(x), \mathcal{I}_{w}) \neq \emptyset,\\
        \text{undefined}, & \text{otherwise.}
    \end{array}\right.
\end{aligned}
\end{gathered}
\]
    \caption{Semantics of weighted $\Sig$-sentences.}
    \label{fig:wSem}
\end{figure}
For the undefined value $\srzero^{-1}$ we use $\srzero$; here, $\supp_{\odot}(\alpha(x), \mathcal{I}_{w})$ is the \emph{support} of $\alpha(x)$ w.r.t. $\mathcal{I}_{w}$ and $\odot \in \{\srsplus, \srstimes\}$, defined as
\[
\supp_{\odot}(\alpha(x), \mathcal{I}_{w}) \;=\; \{\xi \in \sorting(s(x)) \mid \llbracket \alpha(\xi) \rrbracket^{\Sig}_{\mathcal{R}}(\mathcal{I}_{w}) \neq e_{\odot}\},
\]
i.e., the elements $\xi$ in the range of $x$ with a non-neutral value $\llbracket \alpha(\xi) \rrbracket_{\mathcal{R}}^{\Sig}(\mathcal{I}_w)$ w.r.t. $\odot$.
\end{define}
Weighted HT Logic is a generalisation of HT Logic in the following sense:
\begin{proposition}[Generalisation]
Let $\phi$ be a $\Sig$-sentence and $\mathcal{I}_{w}$ be a pointed $\Sig$-HT-interpretation. Then, for the weighted $\Sig$-sentence $\alpha$ over the Boolean semiring $\mathbb{B}$, obtained from $\phi$ by replacing $\bot, \vee, \wedge, \rightarrow, \exists, \forall$ by $0, \splus, \stimes, \rightarrow_{\mathbb{B}}, \bplus, \btimes$, respectively, we have $\llbracket \alpha \rrbracket_{\mathbb{B}}^{\Sig} (\mathcal{I}_{w}) = 1$ iff $\mathcal{I}_{w} \models_{\Sig} \phi$.
\end{proposition}
The proof of the equivalence of $\rightarrow$ and $\rightarrow_{\mathcal{R}}$ works for arbitrary semirings $\mathcal{R}$: for $\Sig$-formulas $\phi, \psi$ the weighted formulas $\phi \rightarrow \psi$ and $\phi \rightarrow_{\mathcal{R}} \psi$ are equivalent. Thus, we can drop $\mathcal{R}$ from $\rightarrow_{\mathcal{R}}$.

Apart from being an HT Logic, the main difference between ours and the Weighted Logic introduced in~\cite{droste2007weighted} is that we allow for the additional connectives $\invplus, \invtimes{}$, and $\rightarrow$ and that ours is first-order over infinite domains instead of second-order over finite words.

Defining a reasonable semantics for the case of infinite support seems challenging in general. For example, $\mathbb{Q}$ is not closed under taking the limit of converging sequences and even in $\mathbb{R}$ not every infinite sum of numbers converges. For $\omega$-continuous semirings such as $\mathbb{N}_{\infty}$, where both conditions above are satisfied, a definition would be possible
(but omitted here).

Intuitively, weighted formulas specify calculations over semirings depending on the truth of formulas. The quantifier $\bplus$ allows us to aggregate the values of weighted formulas for all variable assignments using $\srsplus$ as the aggregate function.
\begin{example}[cont.]
\label{ex:minagg}
The semantics of $\bplus X\; p(X)\stimes X$ over $\mathcal{R}_{\rem{max}}$ is the maximum value $x$ s.t. $p(x)$ holds. As $\llbracket p(x) \stimes x \rrbracket_{\mathcal{R}_{\max}}^{\Sig} (\mathcal{I}_{w}) \neq \srzero = -\infty$ iff $p(x) \in \mathcal{I}^{w}$, we see that for finite $\mathcal{I}^{w}$
\begin{align*}
    \llbracket \bplus X \;p(X) \stimes X \rrbracket_{\mathcal{R}_{\max}}^{\Sig} (\mathcal{I}_{w}) & \;=\; \max\{\llbracket p(x) \stimes x \rrbracket_{\mathcal{R}_{\max}}^{\Sig} (\mathcal{I}_{w}) \mid p(x) \in \mathcal{I}^{w}, x \in \mathbb{Q} \} \\
   & \;=\; \max\{0 + x \mid p(x) \in \mathcal{I}^{w}, x \in \mathbb{Q}\} \;=\; \max\{ x \mid p(x) \in \mathcal{I}^{w}\}.
\end{align*}
\end{example}
The semantics of weighted formulas is multi-valued in general. In order to return to the Boolean semantics for programs, we define algebraic constraints, which are (in)equations between a semiring value and a weighted formula.
\begin{define}[Algebraic Constraints] Let $\Sig = \langle \mathcal{D}, \mathcal{P}, \mathcal{X}, \mathcal{S}, \sorting \rangle$ be a signature. An \emph{algebraic constraint} is an expression $k \sim_{\mathcal{R}} \alpha$ or $x \sim_{\mathcal{R}} \alpha$, where $(\mathcal{R}, >)$ is an ordered semiring, $\alpha$ is a weighted $\Sig$-formula over $\mathcal{R}$, $k \in R$, $x \in \mathcal{X}, \sorting(s(x)) \subseteq R$ and $\sim \in \{>, \geq, =, \leq, <, \not >, \not \geq, \neq, \not\leq, \not <\}$. 

A sentence $k \sim_{\mathcal{R}} \alpha$ is satisfied w.r.t.\ $\mathcal{I}_{w}$, if
\begin{align*}
    \mathcal{I}_{w} \models_{\Sig} k \sim_{\mathcal{R}} \alpha &\iff k \sim \llbracket \alpha \rrbracket^{\Sig}_{\mathcal{R}} (\mathcal{I}_{w'}) \text{ for all }w' \geq w.
\end{align*}
\end{define}
The syntax of $\Sig$-formulas in Section~\ref{sec:prel} is extended to include algebraic constraints in 
(\ref{eq:sigma-formula}) as a further case. The definitions of satisfaction (Defn.~\ref{def:unweighted}) and equilibrium model (Defn.~\ref{def:eq-model}) are amended in the obvious way. However, as the semantics of weighted formulas is undefined for infinite supports, there are two variants of interpreting the condition $\mathcal{I}' \subsetneq \mathcal{I}: (\mathcal{I}', \mathcal{I}, H) \not\models_{\Sig} \phi$ in the definition of equilibrium models. If we adopt that $\not\models_{\Sig}$ holds when the semantics is undefined, we end up with \emph{weak} equilibrium models, otherwise with \emph{strong} equilibrium models.

To verify that the semantics of algebraic constraints is in line with the intuition of HT logic, we show that the persistence property is maintained for sentences that include algebraic constraints.
\begin{proposition}[Persistence]
For any $\Sig$-sentence $\phi$ and $\Sig$-HT-interpretation $(\mathcal{I}^{H}, \mathcal{I}^{T})$, it holds that $\mathcal{I}_{H} \models_{\Sig} \phi$ implies $\mathcal{I}_{T} \models_{\Sig} \phi$.
\end{proposition}
Having established that the semantics behaves as desired, we formally define programs that can contain algebraic constraints in terms of a fragment of the logic over \emph{semiring signatures}.
\begin{define}[Semiring Signature]
A signature $\Sig = \langle \mathcal{D}, \mathcal{P}, \mathcal{X}, \mathcal{S}, \sorting \rangle$ is a \emph{semiring signature} for semirings $\mathcal{R}_1, \dots, \mathcal{R}_n$, where $\mathcal{R}_i = \langle R_i, \srsplus_i, \srstimes_i, e_{\srsplus_i}, e_{\srstimes_i} \rangle$, $i=1,\ldots,n$, if 
\; (i) $\mathcal{S}$ is $2^{\{1, \dots, n\}}$,\; (ii) $\mathcal{D}$ contains $R_i$,  for all $i = 1, \dots, n$, and
\; (iii) $\sorting: \mathcal{S} \rightarrow 2^{\mathcal{D}}$ maps $\{i_1, \dots, i_m\}$ to $\bigcap_{j = 1}^{m} R_{i_j}$.
\end{define}
Intuitively, if a variable $x$ has sort $\{i_1, \dots, i_m\}$,  then we only want to quantify over those domain-values that are in every semiring $\mathcal{R}_{i_1}$ to $\mathcal{R}_{i_m}$. Imagine for example that a variable $x$ is used as a placeholder for a semiring value in two algebraic constraints, one over $\mathbb{N}$ and one over $\mathbb{Q}$. Then it only makes sense to quantify over domain-values that are contained in $\mathbb{N}$.
\begin{define}[\AC-Rules, \AC-Programs]
Let $\Sig = \langle \mathcal{D}, \mathcal{P}, \mathcal{X}, \mathcal{S}, \sorting \rangle$ be a semiring signature for $\mathcal{R}_1, \dots, \mathcal{R}_n$. Then an \AC-program is a set of \AC-rules of the form
\begin{align}
r = H(r) \leftarrow B(r) = \phi \leftarrow \psi_1, \dots, \psi_n, \neg \theta_1, \dots, \neg \theta_m, \label{eq:ruleFormat2}
\end{align}
where each $\phi, \psi_i$ and $\theta_j$ is either a $\Sig$-atom or an algebraic constraint over $\mathcal{R}_i$ for some $i = 1, \dots, n$, in which no quantifiers or nested constraints occur. Furthermore, we require for each variable $x$ occurring in $r$ that $i \in s(x)$ iff $x$ occurs in place of a value from the semiring $\mathcal{R}_i$.
\end{define}
\begin{example}[Rules]
\label{ex:rules}
The following are examples of \AC-rules:
\begin{align}
    \rem{loc\_sum}(Y) &\leftarrow Y =_{\mathbb{Q}} \rem{ind}(I)\stimes \rem{loc\_weight}(I,W)\stimes W\label{rule:intake}\\
    \rem{glob\_sum}(Y) &\leftarrow \rem{glob\_weight}(W), Y =_{\mathbb{Q}} \rem{ind}(I)\stimes W \label{rule:outtake}
\end{align}
\end{example}
Note that in \AC-rules quantifiers occur neither in weighted nor in unweighted formulas. 
Variables are quantified implicitly, depending on their scope defined as follows.
\begin{define}[Local \& Global]
A variable $x$ that occurs in an \AC-rule $r$ is \emph{local}, if it occurs in $r$ only in weighted formulas, and \emph{global} otherwise.
A rule or program is   \emph{locally} (resp.\ \emph{globally}) \emph{ground}, if it has no local (resp.\ global) variables.
\end{define}
\begin{example}[cont.]
In the previous example $Y$ and $I$ are respectively global and local in both rules, whereas $W$ is local in rule \cref{rule:intake} and global in rule \cref{rule:outtake}.
\end{example}
We then quantify global variables universally and local variables ``existentially" (i.e. using $\bplus$).
\begin{define}[Program and Rule Semantics]
Let $r$ be an \AC-rule of the form \cref{eq:ruleFormat2} that contains global variables $x_1, \dots, x_k$. Its semantics is that of the $\Sig$-formula
\begin{align*}
    \forall x_1, \dots, x_k \; (B_{\wedge}(r) \rightarrow \phi)^{\Sigma},\quad \text{ where } B_{\wedge}(r)
    =\psi_1\wedge \dots\wedge \psi_n\wedge \neg \theta_1\wedge \dots\wedge \neg \theta_m 
\end{align*}
and $(\cdot)^{\Sigma}$ replaces every weighted formula $\alpha$ with local variables $y_1, \dots, y_l$ by $\bplus y_1, \dots, y_l \;\alpha$.
\end{define}
\begin{example}[cont.]
Consequently, the \AC-rules from above correspond to the formulas
\begin{align*}
    \forall Y\; (Y =_{\mathbb{Q}} \bplus I \bplus W\;\rem{ind}(I)\stimes \rem{loc\_weight}(I,W)\stimes W)\rightarrow \rem{loc\_sum}(Y)\\
    \forall Y\forall W\;\rem{glob\_weight}(W) \wedge (Y =_{\mathbb{Q}}\bplus I\; \rem{ind}(I)\stimes W)\rightarrow \rem{glob\_sum}(Y).
\end{align*}
We see that rule \cref{rule:intake} calculates the sum over all indices $\{ i \mid \rem{ind}(i)\}$ weighted locally with $w$ when $\rem{loc\_weight}(i,w)$ holds. Rule \cref{rule:outtake} calculates the sum over all indices $\{i \mid \rem{ind}(i)\}$ where all of them are weighted with the same weight $w$ when $\rem{glob\_weight}(w)$ holds. 
\end{example}
Note that we strongly restricted the weighted formulas that are allowed in \AC-programs. The quantifier $\btimes$ and nested algebraic constraints are unavailable and $\bplus$ quantifiers can only occur as a prefix. Removing these restrictions would lead to a much higher complexity. Already constraint evaluation would be PSPACE-hard for any non-trivial semiring. In addition, our choice allows us to keep the syntax of \AC-programs closer to the one of other programs with constraints.

In the sequel, we drop \AC\;from \AC-rules and \AC-programs if no ambiguity arises.

\section{Constructs in ASP(\AC)\;and in other formalisms}
\label{sec:comp}

\begin{table*}[t]
    \centering
    \renewcommand{\arraystretch}{1.0}
    \begin{tabular}{|l|c|c|}
        \toprule
        Construct & ASP(\AC) & Others \\
        \midrule
        Nested Expressions & $1 =_{\mathbb{B}} \alpha \leftarrow 1 =_{\mathbb{B}} \beta$ & $\alpha \leftarrow \beta$ \\
        Aggregates & $T \sim_{\mathbb{Q}} (p(X)\splus q(X))\stimes X$ & $T \sim \rem{sum}\{ X: p(X), X : q(X)\}$ \\
        Choice  & $ k \leq_{\mathcal{R}}^{c} \neg\neg q(X,W) \stimes (q(X,W) \rightarrow p(X)) \stimes W $ & $k \leq \{p(X) : q(X,W) = W\} $ \\
        Minimised Choice & $ k \leq_{\mathcal{R}} \neg\neg q(X,W) \stimes (q(X,W) \rightarrow p(X)) \stimes W $ & n/a \\
        Value Guess & $k \leq_{\mathcal{R}} \rem{val}(X) \stimes X$ & $k \leq \rem{val}$ (CP+ASP) \\
        Arithmetics & $X =_{\mathbb{Q}} Y \stimes \invtimes Z$, $s \geq_{2^{A}} X\splus Y$ & $X = Y \div Z$, $s \supseteq X \cup Y$ \\
        \bottomrule
    \end{tabular}
    \caption{Constructs expressible in ASP(\AC) and how they are expressed in other formalisms.}
    \label{tab:features}
\end{table*}

We consider
several constructs that we can express in ASP(\AC)\;and relate them to constructs known from previous extensions of ASP; a summary is given in Table~\ref{tab:features}.

\paragraph{Nested Expressions}
The logic programs with arbitrary propositional formulas defined in \cite{lifschitz1999nested} are modelled simply using constraints over the Boolean semiring $\mathbb{B}$. As a special case, this shows the expressibility of disjunctive logic programs using $1 =_{\mathbb{B}} a_1 \splus \dots \splus a_n \leftarrow B(r)$.

\paragraph{Conditionals}
Cabalar et al.~\myciteyear{cabalar2020uniform} defined two semantics for conditionals  $s = (s'|s'': \phi)$, where $s', s''$ are terms and $\phi$ is a (quantifier-free) formula. They are named \emph{vicious circle (vc)} and \emph{definedness (df)}, respectively. Given an interpretation $(\mathcal{I}^{H}, \mathcal{I}^{T})$, 
\begin{align*}
    vc_{\mathcal{I}_w}(s) =& \left\{\begin{array}{cl}
        s', & \text{if } \mathcal{I}_{w} \models_{\Sig} \phi, \\
        s'', & \text{if } \mathcal{I}_{w} \models_{\Sig} \neg \phi,\\
        \text{undefined}, & \hfill\text{otherwise.}
    \end{array}\right.  &
    df_{\mathcal{I}_w}(s) =& \left\{\begin{array}{cl}
        s', & \text{if } \mathcal{I}_{w} \models_{\Sig} \phi, \\
        s'', & \text{otherwise.}
    \end{array}\right.
\end{align*}
Syntax and semantics of weighted formulas could be readily extended to include these constructs. We present instead an alternative evaluation of conditionals as formulas $s' \stimes \phi \splus s'' \stimes \neg \phi$. Then 
\begin{align}
    \llbracket s' \stimes \phi \splus s'' \stimes \neg \phi \rrbracket_{\mathcal{R}}^{\Sig}(\mathcal{I}_w) = \left\{\begin{array}{cl}
        s', & \text{if } \mathcal{I}_{w} \models_{\Sig} \phi, \\
        s'', & \text{if } \mathcal{I}_{w} \models_{\Sig} \neg \phi,\\
        \srzero, & \hfill \text{otherwise.}
    \end{array}\right. \label{cond-semantics}
\end{align}
That is, when neither $\phi$ nor $\neg \phi$ is satisfied, we end up with the neutral element $e_{\oplus}$. 
Consider the following rules $r_1$ and $r_2$: 
\begin{align*}
    r_1 = p \leftarrow \top = (\top \mid \bot : p) \vee \top  & & &r_2 = p \leftarrow \top = (\top \mid \top : p).
\end{align*}
According to \emph{vc} resp.\ \emph{df}, they are equivalent: under \emph{vc}, both have no stable model while under \emph{df} both have the stable model $\{p\}$. We may expect that $r_1$ has the stable model $\{p\}$ as the formula $s \vee \top$ is equivalent to $\top$ regardless of the value of $s$. Therefore, the value of $p$ should not influence the truth of the body of $r_1$. On the other hand, the value of the conditional in $r_2$ influences the truth of the body of $r_2$ and it depends on $p$. Therefore, if $\{p\}$ were a stable model of $r_2$, we would arguably derive $p$ using the truth value of $p$. Accordingly, we may expect that $r_2$ does not have a stable model. These expectations align with the semantics for $r_1$ and $r_2$ from (\ref{cond-semantics}) above.
This evaluation combines the ideas behind the 
\emph{vc} and the \emph{df} principle: the value of a conditional is always defined, but the vicious circle of deriving $p$ by the truth value of $p$ is avoided.

Apart from that, 
we may express \emph{vc} and \emph{df} in our formalism: adding the constraint $1 =_{\mathbb{B}} \phi \splus \neg \phi$ to a rule using a conditional on $\phi$ as in (\ref{cond-semantics}) corresponds to \emph{vc} semantics. The constraint enforces that when the rule fires, $\phi$ either must be false already at world $T$, or when $\phi$ is true at $T$ then it must also be forced to hold at world $H$ by the rest of the program. If this is not the case, then $\mathcal{I}_{T} \models 1 =_{\mathbb{B}} \phi \splus \neg \phi$ but $\mathcal{I}_{H} \models 0 =_{\mathbb{B}} \phi \splus \neg \phi$. Thus, in this case, any rule containing this constraint in the body is trivially satisfied.
Furthermore, provided the addition in $\mathcal{R}$ is invertible, we can use $\phi \stimes s' \splus (\srone \splus - \phi)\stimes s''$ to capture \emph{df}  
(cf.\ Appendix for more details and discussion).
Summarising, the possibility to express \emph{vc} and \emph{df} as well as to define other semantics of conditionals exemplifies the power of FO-WHT Logic and \AC-programs.

\paragraph{Constraints in the head for guessing}
\label{sec:headC}
In many 
ASP extensions,
constraints in rule heads and rule bodies behave differently;
in heads, 
they are used as \emph{choice constraints}. Consider for example the rule
$
10\{p(X) : q(X)\}.
$
in lparse syntax.
Any interpretation s.t. $p(x)$ holds for ten or more values $x$ can be stable. In order to express this constraint in our semantics, we need to take care of two aspects. The first one is that the above constraint only supports $p(x)$ for $x$ s.t. $q(x)$ was already derived in another way. If we simply use the rule 
$
10 \leq_{\mathbb{N}} p(X) \stimes q(X)  \leftarrow
$
it can also derive $q(x)$ instead of using it as a precondition. We can however use instead the formula $\neg\neg q(X)\stimes (q(X)\rightarrow p(X))$ to achieve the desired effect. 
More generally, we use the pattern 
\begin{equation}
\label{eq:pattern}
\alpha(X,W) = \neg \neg q(X,W) \stimes (q(X,W) \rightarrow p(X)) \stimes W.
\end{equation}
Abstractly, it ensures that $p(x)$ can only be asserted for $x$ s.t.\ we already know that $q(x,w)$ holds, so we cannot ``invent'' new constants. This can be seen as follows. Assume $\mathcal{I}$ contains $q(x,w), p(x)$. Then $\llbracket \alpha(X) \rrbracket_{\mathbb{N}} (\mathcal{I}, \mathcal{I}, T)$ is equal to $\llbracket \alpha(X) \rrbracket_{\mathbb{N}} (\mathcal{I}\setminus\{q(x,w), p(x)\}, \mathcal{I}, H)$  but unequal to \mbox{$\llbracket \alpha(X) \rrbracket_{\mathbb{N}} (\mathcal{I}\setminus\{p(x)\}, \mathcal{I}, H)$}. The variable $W$ assigns the addition of $p(x)$ a weight $w$.

Secondly, given the rule
$
10 \leq_{\mathbb{N}} \neg\neg q(X)\stimes (q(X)\rightarrow p(X)) \leftarrow
$
only interpretations that assert $p(x)$ for exactly ten elements $x$ can be stable. While such \emph{minimised constraints} are useful in a different context, we also need to be able to specify choice constraints in our language. This can be achieved naturally without extending the semantics of our language, by introducing a syntactic shorthand $k \sim_{\mathcal{R}}^{c} \alpha$ for \emph{algebraic choice constraints} in rule-heads. We define that
\begin{align}
r = k \sim_{\mathcal{R}}^{c} \alpha \leftarrow B(r) \textrm{ ~~~stands for~~~ } 
k \sim_{\mathcal{R}} \alpha \leftarrow B(r) \text{ ~and~ } {X =_{\mathcal{R}} \alpha \leftarrow X =_{\mathcal{R}} \alpha^{\neg\neg}, B(r)},
\end{align}
\noindent where $\alpha^{\neg\neg}$ is obtained from $\alpha$ by adding $\neg\neg$ in front of each atom $p(\overline{x})$. These algebraic choice constraints behave as expected of choice constraints. (The next proposition considers only globally ground rules in order to decrease the amount of syntactic noise.)
\begin{proposition}[Choice Semantics]
For any $\Sig$-HT-interpretation $(\mathcal{I}^{H},\mathcal{I}^{T})$ and any rule $r= k \sim_{\mathcal{R}}^{c} \alpha \leftarrow B(r)$, it holds that 
$\mathcal{I}_{H} \models_{\Sig} r$ iff
 (i) $\mathcal{I}_{H} \models_{\Sig} k \sim_{\mathcal{R}} \alpha \leftarrow B(r)$ and (ii) $\mathcal{I}_{H} \models_{\Sig} (B_{\wedge}(r))^{\Sigma}$ implies $\llbracket (\alpha)^{\Sigma} \rrbracket_{\mathcal{R}}^{\Sig}(\mathcal{I}_{T}) = \llbracket(\alpha)^{\Sigma}\rrbracket_{\mathcal{R}}^{\Sig} (\mathcal{I}_{H})$.
\end{proposition}
\begin{proof}[Proof (sketch)]
Any satisfying interpretation has to satisfy both $k \sim_{\mathcal{R}} \alpha \leftarrow B(r)$ and $X =_{\mathcal{R}} \alpha \leftarrow X =_{\mathcal{R}} \alpha^{\neg\neg}, B(r)$. The first of the two says that the minimised constraint has to be satisfied when $B(r)$ is satisfied and corresponds to $(i)$. The second says that when $B(r)$ is satisfied and $\alpha^{\neg\neg}$ has value $X$ then also $\alpha$ needs to have value $X$. Since the value of $\alpha^{\neg\neg}$ is the value of $\alpha$ under $\mathcal{I}_{T}$ the second rule corresponds to $(ii)$.
\end{proof}
Choice constraints are already well-known from previous ASP extensions, so we do not explain them in more detail. The usefulness of the novel minimised choice constraints is demonstrated in the following example.

\begin{example}[Integer Subset Sum]
Consider the following variation of the Subset Sum Problem: Given a set $S \subseteq \mathbb{Z}$ and two bounds $l, u \in \mathbb{Z}$, determine a $\subseteq$-minimal solution $S' \subseteq S$ such that $l \leq \sum_{x \in S'} x \leq u$.
When $\rem{s}(x)$ holds for $x \in S$, we can use the \AC-rules
\begin{align*}
l \leq_{\mathbb{Z}} \neg \neg \rem{s}(X) \stimes (\rem{s}(X) \rightarrow \rem{in}(X)) \stimes X &\leftarrow & &\text{and} & &u \geq_{\mathbb{Z}} \neg \neg \rem{s}(X) \stimes (\rem{s}(X) \rightarrow \rem{in}(X)) \stimes X \leftarrow.
\end{align*}
For every equilibrium model $\mathcal{I}$ the set $S' = \{x \mid \rem{in}(x) \in \mathcal{I}\}$ is a $\subseteq$-minimal solution and for every $\subseteq$-minimal solution there exists an equilibrium model. When using choice constraints, i.e. replacing $\sim_{\mathbb{Z}}$ by $\sim^{c}_{\mathbb{Z}}$, the program still obtains solutions, but not only $\subseteq$-minimal ones.
\end{example}


\paragraph{Aggregates}
As can be seen in Example \ref{ex:minagg}, we can model aggregates whose aggregation function is the addition of some semiring. This restriction is mild in practice: The aggregates $\rem{min},$ $\rem{max},$ $\rem{sum},$ $\rem{count}$ are expressible using a single algebraic constraint. $\rem{times}$ and $\rem{avg}$ are expressible using multiple algebraic constraints (e.g. $\rem{avg}$ is $\rem{sum}$ divided by $\rem{count}$).

\paragraph{Value Guessing and Arithmetic Operators} Value guessing and arithmetic operators are especially used in combinations of ASP and CP \cite{lierler2014relating}. We can guess a value from a semiring, perform arithmetic operations over semirings and evaluate (in)equations on the results. Again, we are mildly restricted as only semiring operations are available.
\section{Provenance}
\label{sec:prov}
Green at al.~\myciteyear{green2007provenance} introduced a semiring-based semantics that is capable of expressing bag semantics, why-provenance and more. 
For positive logic programs, their semantics over a semiring $(R, \srsplus, \srstimes, \srzero, \srone)$ is as follows: the label of a query result $q(\overline{x})$ is the sum (using $\srsplus$) of the labels of derivation trees for $q(\overline{x})$, where the label of a derivation tree is the product (using $\srstimes$) of the labels of the leaf nodes (i.e. extensional atoms). As the number of derivation trees may be countably infinite, Green et al. used $\omega$-continuous semirings such as $\mathbb{N}_{\infty}$ that allow to have countable sums.

\begin{example}[Bag Semantics]
\label{ex:trans}

For ease of exposition, consider the propositional program
\begin{align*}
    r_1{:}\; b &\leftarrow e_1, e_2&
    r_2{:}\; b &\leftarrow e_1 & 
    r_3{:}\; c &\leftarrow e_2, b& 
    r_4{:}\; c &\leftarrow c, c
\end{align*}
over $\mathbb{N}_{\infty}$ (i.e. with bag semantics) and the extensional database (edb) $\{(e_1,2),$ $(e_2,0)\}$. 
The label of $b$ under bag semantics is $2 + 0\cdot 2 = 2$. Here $2$ corresponds to the derivation from $r_2, (e_1,2)$ and $0\cdot 2 $ to the derivation from $r_1, (e_1,2), (e_2,0)$. The label of $c$ is 0 as it can only be derived using $e_2$.
\end{example}
We can model the semiring semantics in our formalism, by allowing operations over countable supports $\supp_{\odot}(\alpha(x), \mathcal{I}_{w})$ for $\omega$-continuous semirings. Over $\mathbb{N}_{\infty}$ they always have the value $\infty$.
\begin{example}[cont.]
The following \AC-program calculates the provenance semantics over $\mathbb{N}_{\infty}$ for the above positive logic program, depending on the edb:
\newcommand{\lVar}[1]{\ensuremath{#1^{*}}}
\begin{align}
    1 =_{\mathbb{B}} p(b,1,2,X) \stimes d(b,2)    \leftarrow & p(e_1,1,X_1), p(e_2,1,X_2), X =_{\mathbb{N}_{\infty}} X_1 \splus X_2\\
    1 =_{\mathbb{B}} p(b,2,1,X) \stimes d(b,1)       \leftarrow & p(e_1,1,X)\\
    1 =_{\mathbb{B}}p(c,3,V,X) \stimes  d(c,V)       \leftarrow & p(e_2,1,X_1), p(b,V_1,X_2), V =_{\mathbb{N}_{\infty}} V_1 \splus 1, X =_{\mathbb{N}_{\infty}} X_1 \splus X_2\\
    1 =_{\mathbb{B}}p(c,4,V,X) \stimes d(c,V)        \leftarrow & p(c, V_1, X_1), p(c, V_2, X_2), V =_{\mathbb{N}_{\infty}} V_1 \splus 1, X =_{\mathbb{N}_{\infty}} X_1 \splus X_2\\
    1 =_{\mathbb{B}}p(A,V,X)                        \leftarrow & d(A,V), X =_{\mathbb{N}_{\infty}} p(A,I,V,\lVar{X}) \stimes \lVar{X}\label{eq:sumP}\\
    1 =_{\mathbb{B}}f(A,X)                          \leftarrow & d(A,V), X =_{\mathbb{N}_{\infty}} p(A, \lVar{V}, \lVar{X}) \stimes \lVar{X}\label{eq:sumF}
\end{align}
\label{fig:provenance}
Here 
$p(A,V,X)$ represents that $X$ is the sum of all labels of derivation trees for $A$ having exactly $V$ many leaf nodes. We obtain this value first for all derivation trees that apply 
rule $r_i$ last, in $p(A,i,V,X)$, and sum them up in rule \cref{eq:sumP}. Similarly the final provenance value is obtained as the sum over the provenance values for each number of leaf nodes $\lVar{V}$ in rule \cref{eq:sumF}; 
$d(A,V)$ 
says that there is a derivation tree of $A$ using $V$ leaf nodes and ensures safety (see next section).
\end{example}
We can apply this strategy in general: Even for a non-ground positive logic program we can give an \AC-program that computes the provenance semantics. This can be achieved in a similar fashion as in the example above. Details can be found in the appendix. Exploring extensions of Green et al.'s semantics for the provenance of negated formulas remains for future work.
\section{Language Aspects}
\label{sec:lang}
\paragraph{\bf Domain Independence and Safety}

We need to restrict ourselves to programs that are well behaved, i.e. independent of the domain they are evaluated over. 
\begin{example}
\label{ex:domainInd}
Consider the weighted formula
$
\alpha = \bplus x\;\neg q(x),
$
which counts the elements $d$ in the domain s.t. $q(d)$ does not hold. It is easy to see that if we consider the semantics using the same interpretation but over different domains (or rather signatures) it can vary.
\end{example}
We are interested in formulas that do not exhibit this kind of behaviour, formalised as:
\begin{define}[Domain Independence]
A sentence $\phi$ (resp.\ weighted sentence $\alpha$ over semiring $\mathcal{R}$) is \emph{domain independent}, if for every two semiring signatures $\Sig_i = \langle \mathcal{D}_i, \mathcal{P}, \mathcal{X}, \mathcal{S}, \sorting_i \rangle (i = 1,2)$ s.t.\ $\phi$ is a $\Sig_i$-formula (resp. $\alpha$ is a weighted $\Sig_i$-formula) for $i = 1,2$ and every $\mathcal{I}_w = (\mathcal{I}^{H}, \mathcal{I}^{T}, w)$ that is a pointed $\Sig_i$-HT-interpretation for $i = 1,2$ it holds that 
\[
\mathcal{I}_{w} \models_{\Sig_1} \phi\text{ iff }\mathcal{I}_{w} \models_{\Sig_2} \phi \quad\text{ (resp. }\llbracket \alpha \rrbracket_{\mathcal{R}}^{\Sig_1}(\mathcal{I}_{w}) = \llbracket \alpha \rrbracket_{\mathcal{R}}^{\Sig_2}(\mathcal{I}_{w})\text{).}
\]
\end{define}
We restrict ourselves to a fragment of weighted formulas. Intuitively, we need to ensure that every variable $X$ in $\alpha(\overline{X})$ is bound by a positive occurrence of a predicate $p(X)$.
\begin{define}[Syntactic Domain Independence]
A weighted formula $\alpha(\overline{X})$ over a semiring $\mathcal{R}$ is \emph{syntactically domain independent} w.r.t. $\overline{X}$, if it is constructible following
\begin{align*}
\phi(\overline{X}) &::= \bot \mid p(\overline{X}) \mid \neg \neg \phi(\overline{X})\mid \phi(\overline{X}) \vee \phi(\overline{X}) \mid \phi(\overline{Y}) \wedge \phi(\overline{Z}) \mid \phi(\overline{X}) \wedge \psi(\overline{X'}),\\
\alpha(\overline{X}) &::= k \mid \phi(\overline{X}) \mid \neg \neg \alpha(\overline{X})\mid \alpha(\overline{X}) \splus \alpha(\overline{X}) \mid \alpha(\overline{Y}) \stimes \alpha(\overline{Z}) \mid \alpha(\overline{X}) \stimes \beta(\overline{X'}) \mid \invplus \alpha(\overline{X}) \mid \invtimes \alpha(\overline{X}),
\end{align*}
where $k \in R$, $p(\overline{X})$ is an atom, $\psi(\overline{X'})$ ($\beta(\overline{X'})$) is any (weighted) formula, $\overline{X'} \subseteq \overline{X}$ and $\overline{Y} \cup \overline{Z}  = \overline{X}$.
\end{define}
\begin{example}[cont.]
While $\neg q(Y)$ from Example~\ref{ex:domainInd} is not syntactically domain independent w.r.t.\ $Y$, the formula $p(Y) \stimes \neg q(Y)$, which counts the number $d$ s.t. $p(d)$ holds but not $q(d)$, is. It can be constructed using $\alpha(\overline{X}) \stimes \beta(\overline{X'})$. 
\end{example}
Our syntactic criterion guarantees domain independence.
\begin{theorem}[Formula Domain Independence]
If a formula $\alpha(\overline{X})$ over semiring $\mathcal{R}$ is syntactically domain independent w.r.t. $\overline{X}$, then $\alpha^{\Sigma} = \bplus \overline{X} \; \alpha(\overline{X})$ is domain independent.
\end{theorem}
\begin{proof}[Proof (sketch)]
Invariance of $\supp_{\srsplus}(\alpha(x), \mathcal{I}_{w})$ w.r.t. $\Sig_i$ (or rather $\mathcal{D}_i$) is shown by structural induction. We show the invariance for one interesting case, namely $\alpha = \alpha_1(x) \stimes \alpha_2$.
    Note that:
    \begin{align*}
        \{\xi \in \sorting_1(s(x)) \mid \llbracket \alpha_1(\xi) \stimes \alpha_2 \rrbracket^{\Sig_1}_{\mathcal{R}}(\mathcal{I}_{w}) \neq e_{\oplus}\}
        \subseteq \{\xi \in \sorting_1(s(x)) \mid \llbracket \alpha_1(\xi) \rrbracket^{\Sig_1}_{\mathcal{R}}(\mathcal{I}_{w}) \neq e_{\oplus}\}
    \end{align*}
    Therefore, we obtain
    \begin{align*}
        & \{\xi \in \sorting_1(s(x)) \mid \llbracket \alpha_1(\xi) \stimes \alpha_2 \rrbracket^{\Sig_1}_{\mathcal{R}}(\mathcal{I}_{w}) \neq e_{\oplus}\}\\
        =& \{\xi \in \{\xi \in \sorting_1(s(x)) \mid \llbracket \alpha_1(\xi) \rrbracket^{\Sig_1}_{\mathcal{R}}(\mathcal{I}_{w}) \neq e_{\oplus}\} \mid \llbracket \alpha_1(\xi) \stimes \alpha_2 \rrbracket^{\Sig_1}_{\mathcal{R}}(\mathcal{I}_{w}) \neq e_{\oplus}\}
    \end{align*}
    Next, we use the induction hypothesis on $\alpha_1(x)$ to obtain
    \begin{align*}
        =& \{\xi \in \{\xi \in \sorting_2(s(x)) \mid \llbracket \alpha_1(\xi) \rrbracket^{\Sig_2}_{\mathcal{R}}(\mathcal{I}_{w}) \neq e_{\oplus}\} \mid \llbracket \alpha_1(\xi) \stimes \alpha_2 \rrbracket^{\Sig_2}_{\mathcal{R}}(\mathcal{I}_{w}) \neq e_{\oplus}\}\\
        =&\{\xi \in \sorting_2(s(x)) \mid \llbracket \alpha_1(\xi) \stimes \alpha_2 \rrbracket^{\Sig_2}_{\mathcal{R}}(\mathcal{I}_{w}) \neq e_{\oplus}\}.
    \end{align*}
    As the semantics of variable-free formulas is domain independent the claim follows.
\end{proof}

Safety of programs is defined as follows.
\begin{define}[Safety]
\label{def:safety}
A program $\Pi$ is safe, if each rule $r\in \Pi$ of form \cref{eq:ruleFormat2} is safe,
i.e.\ fulfills that
\begin{enumerate}[label={\upshape(\roman*)},align=left, widest=iii, leftmargin=*]
\itemsep=0pt
    \item every weighted formula in $r$ is syntactically domain independent w.r.t. its local variables;
    \item for every global variable $X$ there exists some $\beta_i$ s.t.\ (1) $\beta_i$ is an atom and $X$ occurs in it, or
    (2) $\beta_i$ is $X\,{=_{\mathcal{R}}}\, \beta_i'$ and $X$ does not occur in any weighted formula in the body of $r$.\label{enum:safe2}
\end{enumerate}
\end{define}
The restriction in \ref{enum:safe2} that $X$ does not reoccur is necessary to prohibit $p(X) \leftarrow X=_{\mathcal{R}} Y, Y =_{\mathcal{R}} X$. It could however be replaced by a more sophisticated acyclicity condition.
\begin{example}[Safety]
The rules \cref{rule:intake} and \cref{rule:outtake} 
are safe. Without the predicate $d$ the program in Example~\ref{fig:provenance} would not be safe.
\end{example}
\begin{theorem}[Program Domain Independence]
Safe programs are domain independent.
\end{theorem}
\begin{proof}[Proof (sketch)]
Let $\Sig_i = \langle \mathcal{D}_i, \mathcal{P}, \mathcal{X}, \mathcal{S}, \sorting_i \rangle, i = 1,2$ be semiring signatures s.t.\ $\Pi$ is a $\Sig_i$-formula for $i = 1,2$ and
let $\mathcal{I}_{w} = (\mathcal{I}^{H},\mathcal{I}^{T},w)$ be a pointed $\Sig_i$-HT-interpretation for $i =1,2$. 

Let $r \in \Pi$. If $r$ does not contain global variables, the claim is evident. Otherwise assume $r = \forall x_1, \dots, x_n \; \alpha(x_1, \dots, x_n)$. When $\xi_j \in \sorting_1(s(x_j)) \cap \sorting_2(s(x_j))$ ($j = 1, \dots, n$), the semantics of $\alpha(\xi_1, \dots, \xi_n)$ does not depend on $\Sig_i$.
Suppose that $\xi_j \in \sorting_1(s(x_j)) \setminus \sorting_2(s(x_j))$. Then $x_j$ cannot occur in place of a semiring value as for semiring signatures, we have $\sorting_1(s(x_j)) = \bigcap_{j = 1}^{m} R_{k_j} = \sorting_2(s(x_j))$. Therefore $x_j$ has to satisfy item (ii.1) of safety, implying that some
atom $\beta_k$ in the body of $r$ is not satisfied by $\mathcal{I}_{w}$ and hence $\mathcal{I}_{w} \models_{\Sig_i} \alpha(\xi_1, \dots, \xi_n)$ for $i = 1,2$.
\end{proof}
Not every domain independent program is safe. E.g. $p(X) \leftarrow \top =_{\mathbb{B}} q(X)$ is not safe but is equivalent to the safe rule $p(X) \leftarrow q(X)$ since $X$ is a global variable and we can only derive $p(x)$ when $\llbracket q(x) \rrbracket_{\mathbb{B}} = 1$, i.e.\ when $q(x)$ holds. Domain independence is undecidable but safety is sufficient, allows for complex rules like \cref{rule:intake}, \cref{rule:outtake} and those in Example~\ref{fig:provenance}, and is easily checked.

In the rest of the paper, we restrict ourselves to domain independent programs and can therefore remove the annotation $\Sig$ from $\models_{\Sig}$ and $\llbracket \cdot \rrbracket_{\mathcal{R}}^{\Sig}$ and use $\models$ and $\llbracket \cdot \rrbracket_{\mathcal{R}}$ instead. Accordingly, we do not need to specify the signature for \AC-programs $\Pi$ anymore, as any semiring signature $\Sig$ s.t. $\Pi$ is an \AC-program over $\Sig$ suffices.
\paragraph{\bf Program Equivalence}
An additional benefit of HT-semantics is that we are able to characterise strong program equivalence as equivalence in the logic of HT.
\begin{define}[Strong Equivalence]
Programs $\Pi_1$ and $\Pi_2$ are strongly equivalent, denoted by $\Pi_1 \,{\equiv_s}\, \Pi_2$, if for every program $\Pi'$ 
 the equilibrium models of $\Pi_1 \,{\cup}\, \Pi'$ and $\Pi_2 \,{\cup}\, \Pi'$ coincide.
\end{define}
Similar results have already been proven for classical programs with \cite{10.1007/978-3-540-89982-2_46} or without variables \cite{lifschitz2001strongly} and many more.
As with classical programs:
\begin{theorem}
For any $\Pi_1, \Pi_2$ programs, $\Pi_1 \equiv_s \Pi_2$ iff $\Pi_1$ has the same HT-models, i.e. satisfying pointed HT-interpretations, as $\Pi_2$.
\end{theorem}
\begin{proof}[Proof (sketch)]
The direction $\Leftarrow$ is clear. For $\Rightarrow$ we can generalise the proof in \cite{lifschitz2001strongly}, by constructing $\Pi'$, which asserts a subset of the interpretation $\mathcal{I}^{T}$ that is ensured to be stable ($\mathcal{I}^{H}$), and a subset that if partly present is ensured to be fully present ($\mathcal{I}^{T}\setminus\mathcal{I}^{H}$). 

Let $\Pi_1$ and $\Pi_2$ have different HT-models. W.l.o.g. there must be at least one HT-interpretation $(\mathcal{I}^{H}, \mathcal{I}^{T})$ that is an HT-model of $\Pi_1$ but not of $\Pi_2$. 
As in \cite{lifschitz2001strongly} we simply define
\[
\Pi' = \{ p(\overline{x}) \leftarrow \mid p(\overline{x}) \in \mathcal{I}^{H} \} \cup \{p(\overline{x}) \leftarrow q(\overline{y}) \mid p(\overline{x}), q(\overline{y}) \in \mathcal{I}^{T} \setminus \mathcal{I}^{H}\}
\]
Then $\mathcal{I}^{T}$ is an equilibrium model of $\Pi_2\cup \Pi'$, but not of $\Pi_1 \cup \Pi'$ and therefore $\Pi_1$ and $\Pi_2$ are not strongly equivalent.
\end{proof}
Note that since $\mathcal{I}^{H}$ may be infinite, this may result in programs of infinite size. This can be circumvented if auxiliary predicates are allowed in $\Pi'$ (see the Appendix).

\section{Computational Complexity}
\label{sec:complexity}
We consider the computational complexity of the following problems:
\begin{itemize}
    \item Model Checking (MC): Given a safe program $\Pi$ and an interpretation $\mathcal{I}$ of $\Pi$, is $\mathcal{I}$ is an equilibrium model of $\Pi$?
    \item Satisfiability (SAT): Given a safe program $\Pi$, does $\Pi$
      have an equilibrium model?
    \item Strong Equivalence (SE): Given safe programs $\Pi_1,\Pi_2$, are $\Pi_1$ and $\Pi_2$ strongly equivalent?
\end{itemize}
The main factor that complicates these problems is that we may have to evaluate weighted formulas over an arbitrary semiring. If we want to prevent an increase in complexity, then we need to encode the elements of the semiring in some way which allows for efficient calculations and comparison. To this end, we use Efficient Encodability from~\cite{eiter2020wlars}.
\begin{define}[Efficiently Encodable Semiring]
Let $\mathcal{R} = (R, \srsplus, \srstimes, \srzero, \srone)$ be a semiring with strict order $>$ and $e : R \rightarrow \mathbb{N}$ a polynomially computable, injective function. We use $\lVert r \rVert = \log_2(e(r))$ for the length of $e(r)$'s representation. We say $\mathcal{R}$ is \emph{efficiently encoded} by $e$ if
\begin{myitemize}
    \item some $c \in \mathbb{N}$ exists such that for $r_1, r_2 \in R'$ and $\odot \in \{\srsplus, \srstimes\}$ \smallskip\\
    \mbox{\qquad} $\lVert r_1\odot r_2\rVert \leq \lVert r_1 \rVert + \lVert r_2 \rVert + c$~~ 
    and ~~$\max(\lVert -r_1 \rVert, \lVert r_1^{-1} \rVert) \leq \lVert r_1 \rVert + c;$ \hfill
    \rem{\refstepcounter{equation}(\theequation)\label{eq:eff2}}\smallskip
    \item we can compute $e(r_1\odot r_2), e(-r_1), e(r_1^{-1})$ in 
    polynomial time given $e(r_1),e(r_2)$ resp.\ $e(r_1)$;
    \item $r_1 > r_2$ is decidable in time polynomial in $\lVert r_1 \rVert + \lVert r_2 \rVert$.
\end{myitemize}


\end{define}
This restriction is
mild in practice; for example $\mathbb{B},$ $\mathbb{N},$ $\mathbb{Z},$ $\mathbb{Q},$ $\mathcal{R}_{\max},$ $2^{A}$ are efficiently encodable.
\begin{theorem}[Ground Complexity]
For variable-free programs
over efficiently encoded semirings
(i) MC and (propositional) SE are co-NP-complete, and (ii) SAT is $\Sigma^{p}_2$-complete.
\end{theorem}
\begin{proof}[Proof (sketch)]
The hardness parts are inherited from 
disjunctive logic programs \cite{DBLP:journals/csur/DantsinEGV01},
cf.\ Section~\ref{sec:comp}, resp.\ HT-Logic \cite{lifschitz2001strongly}.
The membership parts 
result by guess and check algorithms: for similar bounds as in
ordinary ASP, we just need that $\mathcal{I}_{H} \models k
\sim_{\mathcal{R}} \alpha$ is 
polynomially decidable given $(\mathcal{I}^{H}, \mathcal{I}^{T})$ and 
$k \sim_{\mathcal{R}} \alpha$; as $\mathcal{R}$ is efficiently
encoded, this 
holds.
\end{proof}
The non-ground complexity is significantly higher.
\begin{theorem} [Non-ground Complexity]
\label{theo:non-ground-complexity}
For safe programs over efficiently encoded semirings
(i) MC is in EXPTIME, both co-NP$^{\text{PP}}$-hard and NP$^{\text{PP}}$-hard, and
(ii) SAT and SE are undecidable.
\end{theorem}
\begin{proof}[Proof (sketch)]
(i) Given an interpretation $\mathcal{I}$ as a set of ground atoms, we check $(\mathcal{I}, \mathcal{I}, H) \models r'$ and $(\mathcal{I}', \mathcal{I}, H) \not\models r'$ for each $\mathcal{I}'\subsetneq \mathcal{I}$ and each ground instance $r'$ of a rule $r\in \Pi$ in exponential time. The evaluation of algebraic constraints $k \sim_{\mathcal{R}} \alpha$ is feasible in exponential time, since if $\alpha$ is of the form $\Sigma y_1,\ldots,y_n \alpha'(y_1,\ldots,y_n)$ where $\alpha'$ is quantifier-free, by safety of the program each $y_i$ must occur in some atom $p(\overline{x})$. That is, to evaluate $\alpha$, we only need to consider values $\xi(y_i)$ for $y_i$, $i=1,\ldots,_n$ that occur in the interpretation $\mathcal{I}$. There are exponentially many such $\xi$; for each of them, the value of $\alpha'(\xi(y_1),\ldots,\xi(y_n))$ can be computed in polynomial time given that $\mathcal{R}$ is efficiently encoded, yielding a value $r_\xi$ such that $e(r_\xi)$ occupies polynomially many bits. The aggregation $\Sigma_\xi\, r_\xi$ over all $\xi$ is then feasible in exponential time by the assertion that $\|r_1 \oplus r_2\| \leq \|r_1\|+\|r_2\|+c$ and that $e(r_1 \oplus r_2)$ is computable in polynomial time given $e(r_1),e(r_2)$.

The {(co-)}NP$^{\text{PP}}$-hardness is by a reduction from (co-)E-MAJSAT
\cite{littman1998computational}, which asks whether for a Boolean formula $\phi(x_1, \dots, x_n)$ a partial assignment to $x_1, \dots, x_k$ exists s.t. more than $m = 2^{n-k-1}$ of the assignments to $x_{k+1},\dots, x_n$ satisfy $\phi(\overline{x})$. Then the program 

\smallskip

\centerline{$
v(0) \leftarrow\;; \phantom{asdf} v(1) \leftarrow\;;
    \hphantom{asdf}f \leftarrow v(X_1), \dots, v(X_k), m <_{\mathbb{N}} v(X_{k+1}) \stimes \dots \stimes v(X_n) \stimes \phi(\overline{X})$}

\smallskip


\noindent has an equilibrium model $\{f, v(0), v(1)\}$ if the answer for
E-MAJSAT is yes and an equilibrium model $\{v(0), v(1)\}$ if the
answer is no. 

(ii) The undecidable Mortal Matrix Problem (MMP) asks whether any product of matrices in $X= \{X_1, \dots, X_n\} \subset \mathbb{Z}^{d\times d}$ evaluates to the zero matrix $0_d$ \cite{DBLP:journals/corr/CassaigneHHN14}.
The semiring $(\mathbb{Z}^{d\times d}, +, \cdot, 0_d, 1_d)$ is efficiently encodable, and the program $\Pi$

\smallskip

\centerline{$p(X_1) \leftarrow\;; \phantom{asd} \dots \phantom{asd}
  p(X_n) \leftarrow\;; \phantom{asdf} \bot \leftarrow \neg p(0_d);\; \phantom{asdf}p(Y) \leftarrow p(Z_1), p(Z_2), Y =_{\mathbb{Z}^{d\times d}} Z_1\stimes Z_2
$}

\smallskip

\noindent has an equilibrium model iff $X$ is a yes-instance of MMP,
as $p(0_d)$ needs to be supported. For undecidability, let $\Pi$ be the program from above and $\Pi' = \Pi\setminus\{\bot \leftarrow \neg p(0)_d)\}$. As $\Pi'$ has no negation, its HT-models are the interpretations $(\mathcal{I}',\mathcal{I})$ where both $\mathcal{I}'$ and $\mathcal{I}$ are closed under the rules of $\Pi'$, sets $S$ such that $p(X_1), \dots, p(X_n) \in S$ and whenever $p(Y), p(Z) \in S$ then also $p(Y*Z) \in S$. 
Similarly, the HT-models of $\Pi$ are the interpretations $(\mathcal{I}',\mathcal{I})$ where
$\mathcal{I}'$ and $\mathcal{I}$ are closed under the rules of $\Pi'$ and in addition $p(0_d) \in \mathcal{I}'$.

Therefore, $\Pi \equiv_s \Pi'$ iff $p(0_d) \in L$, where $L$ is the least set closed under the rules of $\Pi'$, which holds iff the answer for the mortal matrix problem on $X$ is yes.
\end{proof}
As
NP$^{\text{PP}}$ contains the polynomial hierarchy (PH),
this places MC between PH and EXPTIME; stronger assumptions on the encoding $e(r)$ 
allow for PSPACE. In particular, for programs over the canonical
semiring $\mathbb{N}$,
MC is co-NP$^{C}$-complete for $C = $ NP$^{\text{PP}}$
(while SAT and SE are undecidable).
Na\"ive evaluation of $k \sim_{\mathcal{R}} \alpha$ is infeasible in polynomial space, as $\lVert \llbracket \alpha \rrbracket_{\mathcal{R}}(\mathcal{I}_{H})\rVert$ can be exponential in the number of variables in $\alpha$.
We can retain decidability for SAT and SE by limiting value invention,
i.e. constraints 
$X =_{\mathcal{R}} \alpha(\overline{Y})$, and value guessing. For the
latter, we adapt 
domain restrictedness from \cite{niemela1999stable}.
\begin{define}[Domain Restrictedness]
An algebraic constraint is \emph{domain restricted} in variables $\overline{X}$, if it is of the form \\[-6pt]
\centerline{$k \sim_{\mathcal{R}} \neg \neg \alpha(\overline{X}) \stimes (\alpha(\overline{X}) \rightarrow \beta(\overline{X})) \stimes \gamma(\overline{X}),$}\\
\noindent where $\alpha(\overline{X}), \beta(\overline{X})$ are syntactically domain independent and all atoms in $\gamma(\overline{X})$ are locally ground.
\end{define}
Intuitively, only constants ``known'' by predicates in $\alpha$ can be ``transferred'' to predicates in $\beta$, and $\gamma$ assigns a weight to each substitution. The pattern is explained less generally in Section \ref{sec:comp}. Let us call a semiring \emph{computable} if $\srsplus, \srstimes, -, ^{-1}, >$ are computable. Then we obtain:
\begin{theorem}
\label{theo:domain-restricted-complexity}
For safe programs without value invention where all algebraic constraints in rule heads are domain restricted and all semirings are computable, both SAT and SE are decidable. 
\end{theorem}
\begin{proof}[Proof (sketch)]
For this class of programs we can show that they are finitely groundable, i.e. groundable over a finite domain without changing the answer sets, by using only the constants that occur in a program as the domain. Then the ground programs are variable free and since the semirings are computable both SAT and SE are decidable.
\end{proof}
However, prohibiting value invention entirely is 
unnecessarily strong. Weaker restrictions like aggregate stratification \cite{faber2011semantics} or argument restrictedness \cite{lierler2009one} can be adapted to ASP(\AC);
the resulting programs are finitely ground
and decidable.
 
\section{Related Work \& Conclusion}
\label{sec:rela}
A number of related works has already been mentioned above; we concentrate here on highlighting the differences of our approach to others.
\begin{myitemize}
\itemsep=0pt
\item 
{Semiring-based Constraint Logic Programming}
\cite{bistarelli1997semiring},
{aProbLog} 
\cite{kimmig2011algebraic} and our previous work~\myciteyear{eiter2020wlars} also use semiring semantics. 
However, 
Bistarelli et al.\ aimed at semantics for 
CLP with multi-valued interpretations over lattice-like semirings and the other works
aimed at semantics for weighted model counting and model selection over semirings. 
\item {Hybrid ASP} by Cabalar et al.~\myciteyear{cabalar2020uniform,cabalar2020asp}. 
They defined an extension of HT Logic that includes general constraints and multi-valued interpretations 
for handling mixtures of ASP and CP. The approach 
integrates conditionals and aggregates; however, it relies on extra definitions to introduce their semantics
while our semantics can capture the different constructs natively. 
The syntax of constraints (apart
from 
over the reals) is left open, while we provide a uniform 
syntax over any semiring. Moreover, we study domain independence and safety, characterise strong equivalence, and provide complexity results.
\item 
{Nested Formulas with Aggregates} due to Ferraris and Lifschitz~\myciteyear{ferraris2010stable} have semantics similar to that of HT. 
They allow for arbitrary aggregate functions but only over the integers, whereas we allow for arbitrary values 
using semiring operations. While defined, 
the usage of non-ground constraints in rule heads 
was not considered. We can transfer our results and show that both choice and minimised constraints can be encoded and used safely in the formalism.

\item 
{Gelfond-Zhang Aggregates}~\myciteyear{gelfond2014vicious} are semantically different from ours but presumably encodable in ASP(\AC). Their semantics introduced the vicious circle principle to ASP. Regarding expressiveness, aggregates are not allowed in rule-heads but aggregation functions are arbitrary.

\item 
Arbitrary Constraint Atoms due to Son et al.~(2007) 
as well as (monotone) Abstract Constraint Atoms due to Marek et al.~(2006) 
define semantics for constraints abstractly by allowing them to be specified as a set of alternative sets of atoms that need to be satisfied. Naturally, this gives a semantics to arbitrary constraints, however syntactic shorthands are desirable to avoid an exponential blowup of the representation of the constraints. Marek et al.\ focus on monotone constraints, showing that their behaviour can be characterised by fixed-points of a non-deterministic operator. 
\item 
{Weight Constraints with Conditionals} by Niemelä et al.~\myciteyear{niemela1999stable} introduced the well known shorthands $k \leq \{p(X) : q(X,W) = W\}$. Our constraints generalise them to arbitrary semirings.

\item Formalisms for Intensional Functions in ASP as in \cite{bartholomew2019first,cabalar2011functional} define a semantics that allows the definition of functions using ASP. A priori, this differs from our goal aiming at an expressive predicate-based formalism. Nevertheless, Weighted Formulas could be used to specify the values of functions. Semantically we are closer to Cabalar et al.'s
approach, where function values can be undefined and the stability condition is more similar.
\end{myitemize}

\paragraph{Summary \& Outlook} We have seen that algebraic constraints unify many previously proposed constructs for more succinct answer set programs, with low practical restrictions and no increase in the ground complexity. Among other novelties, we can specify whether constraints in rule-heads are minimised or guessed, can explicitly represent values from different sets and give an interesting alternative semantics for conditionals. Overall, the introduced framework opens up new possibilities for expressing programs succinctly and it gives rise to interesting questions.

We currently consider only a fragment of the weighted formulas. It would be interesting to see in the future, if other new and useful constructs can be expressed with a different fragment. Besides this, we want to use the general applicability of HT and Weighted Logic and extend ASP(\AC)\;to other domains, like temporal reasoning.
Furthermore, an in-depth study of suitable conditions for finite groundability and the non-ground complexity in this context are indispensable for our ongoing work on an implementation. 

\section*{Acknowledgement}
Thanks to the reviewers for their constructive comments. This work has been supported by FWF project W1255-N23 and by FFG project 861263. 
\nocite{son2007answer,marek2006logic}

\bibliography{new_tlp2egui.bib}

\begin{thebibliography}{}

\bibitem[\protect\citeauthoryear{Aziz, Chu, and Stuckey}{Aziz
  et~al\mbox{.}}{2013}]{DBLP:journals/tplp/AzizCS13}
{\sc Aziz, R.~A.}, {\sc Chu, G.}, {\sc and} {\sc Stuckey, P.~J.} 2013.
\newblock Stable model semantics for founded bounds.
\newblock {\em Theory Pract. Log. Program.\/}~{\em 13,\/}~4-5, 517--532.

\bibitem[\protect\citeauthoryear{Balai, Gelfond, and Zhang}{Balai
  et~al\mbox{.}}{2013}]{balai2013sparc}
{\sc Balai, E.}, {\sc Gelfond, M.}, {\sc and} {\sc Zhang, Y.} 2013.
\newblock Sparc-sorted asp with consistency restoring rules.
\newblock {\em arXiv preprint arXiv:1301.1386\/}.

\bibitem[\protect\citeauthoryear{Bartholomew and Lee}{Bartholomew and
  Lee}{2019}]{bartholomew2019first}
{\sc Bartholomew, M.} {\sc and} {\sc Lee, J.} 2019.
\newblock First-order stable model semantics with intensional functions.
\newblock {\em Artificial Intelligence\/}~{\em 273}, 56--93.

\bibitem[\protect\citeauthoryear{Bistarelli, Montanari, and Rossi}{Bistarelli
  et~al\mbox{.}}{1997}]{bistarelli1997semiring}
{\sc Bistarelli, S.}, {\sc Montanari, U.}, {\sc and} {\sc Rossi, F.} 1997.
\newblock Semiring-based constraint logic programming.
\newblock In {\em Proc. IJCAI'97}. 352--357.

\bibitem[\protect\citeauthoryear{Cabalar}{Cabalar}{2011}]{cabalar2011functional}
{\sc Cabalar, P.} 2011.
\newblock Functional answer set programming.
\newblock {\em Theory and Practice of Logic Programming\/}~{\em 11,\/}~2-3,
  203--233.

\bibitem[\protect\citeauthoryear{Cabalar, Fandinno, Schaub, and Wanko}{Cabalar
  et~al\mbox{.}}{2020a}]{cabalar2020asp}
{\sc Cabalar, P.}, {\sc Fandinno, J.}, {\sc Schaub, T.}, {\sc and} {\sc Wanko,
  P.} 2020a.
\newblock An {ASP} semantics for constraints involving conditional aggregates.
\newblock {\em arXiv preprint arXiv:2002.06911\/}.

\bibitem[\protect\citeauthoryear{Cabalar, Fandinno, Schaub, and Wanko}{Cabalar
  et~al\mbox{.}}{2020b}]{cabalar2020uniform}
{\sc Cabalar, P.}, {\sc Fandinno, J.}, {\sc Schaub, T.}, {\sc and} {\sc Wanko,
  P.} 2020b.
\newblock A uniform treatment of aggregates and constraints in hybrid {ASP}.
\newblock {\em arXiv preprint arXiv:2003.04176\/}.

\bibitem[\protect\citeauthoryear{Cassaigne, Halava, Harju, and
  Nicolas}{Cassaigne et~al\mbox{.}}{2014}]{DBLP:journals/corr/CassaigneHHN14}
{\sc Cassaigne, J.}, {\sc Halava, V.}, {\sc Harju, T.}, {\sc and} {\sc Nicolas,
  F.} 2014.
\newblock Tighter undecidability bounds for matrix mortality,
  zero-in-the-corner problems, and more.
\newblock {\em CoRR\/}~{\em abs/1404.0644}.

\bibitem[\protect\citeauthoryear{Dantsin, Eiter, Gottlob, and Voronkov}{Dantsin
  et~al\mbox{.}}{2001}]{DBLP:journals/csur/DantsinEGV01}
{\sc Dantsin, E.}, {\sc Eiter, T.}, {\sc Gottlob, G.}, {\sc and} {\sc Voronkov,
  A.} 2001.
\newblock Complexity and expressive power of logic programming.
\newblock {\em {ACM} Comput. Surv.\/}~{\em 33,\/}~3, 374--425.

\bibitem[\protect\citeauthoryear{Droste and Gastin}{Droste and
  Gastin}{2007}]{droste2007weighted}
{\sc Droste, M.} {\sc and} {\sc Gastin, P.} 2007.
\newblock Weighted automata and weighted logics.
\newblock {\em Theor.Comp.Sci.\/}~{\em 380,\/}~1, 69.

\bibitem[\protect\citeauthoryear{Eiter and Kiesel}{Eiter and
  Kiesel}{2020}]{eiter2020wlars}
{\sc Eiter, T.} {\sc and} {\sc Kiesel, R.} 2020.
\newblock Weighted lars for quantitative stream reasoning.
\newblock In {\em Proc. ECAI'20}.

\bibitem[\protect\citeauthoryear{Faber, Pfeifer, and Leone}{Faber
  et~al\mbox{.}}{2011}]{faber2011semantics}
{\sc Faber, W.}, {\sc Pfeifer, G.}, {\sc and} {\sc Leone, N.} 2011.
\newblock Semantics and complexity of recursive aggregates in answer set
  programming.
\newblock {\em Artificial Intelligence\/}~{\em 175,\/}~1, 278--298.

\bibitem[\protect\citeauthoryear{Ferraris}{Ferraris}{2011}]{ferraris2011logic}
{\sc Ferraris, P.} 2011.
\newblock Logic programs with propositional connectives and aggregates.
\newblock {\em ACM TOCL\/}~{\em 12,\/}~4, 25.

\bibitem[\protect\citeauthoryear{Ferraris and Lifschitz}{Ferraris and
  Lifschitz}{2010}]{ferraris2010stable}
{\sc Ferraris, P.} {\sc and} {\sc Lifschitz, V.} 2010.
\newblock On the stable model semantics of first-order formulas with
  aggregates.
\newblock In {\em Proc.\ International Workshop on Nonmonotonic Reasoning
  (NMR'10)}.

\bibitem[\protect\citeauthoryear{Gelfond and Zhang}{Gelfond and
  Zhang}{2014}]{gelfond2014vicious}
{\sc Gelfond, M.} {\sc and} {\sc Zhang, Y.} 2014.
\newblock Vicious circle principle and logic programs with aggregates.
\newblock {\em Theory and Practice of Logic Programming\/}~{\em 14,\/}~4-5,
  587--601.

\bibitem[\protect\citeauthoryear{Green, Karvounarakis, and Tannen}{Green
  et~al\mbox{.}}{2007}]{green2007provenance}
{\sc Green, T.~J.}, {\sc Karvounarakis, G.}, {\sc and} {\sc Tannen, V.} 2007.
\newblock Provenance semirings.
\newblock In {\em Proc. ACM PODS'07}. ACM, 31--40.

\bibitem[\protect\citeauthoryear{Janhunen}{Janhunen}{2018}]{DBLP:journals/ki/Janhunen18}
{\sc Janhunen, T.} 2018.
\newblock Answer set programming related with other solving paradigms.
\newblock {\em {KI}\/}~{\em 32,\/}~2-3, 125--131.

\bibitem[\protect\citeauthoryear{Kimmig, Van~den Broeck, and De~Raedt}{Kimmig
  et~al\mbox{.}}{2011}]{kimmig2011algebraic}
{\sc Kimmig, A.}, {\sc Van~den Broeck, G.}, {\sc and} {\sc De~Raedt, L.} 2011.
\newblock An algebraic prolog for reasoning about possible worlds.
\newblock In {\em Proc. AAAI'11}.

\bibitem[\protect\citeauthoryear{Lierler}{Lierler}{2014}]{lierler2014relating}
{\sc Lierler, Y.} 2014.
\newblock Relating constraint answer set programming languages and algorithms.
\newblock {\em Artificial Intelligence\/}~{\em 207}, 1--22.

\bibitem[\protect\citeauthoryear{Lierler and Lifschitz}{Lierler and
  Lifschitz}{2009}]{lierler2009one}
{\sc Lierler, Y.} {\sc and} {\sc Lifschitz, V.} 2009.
\newblock One more decidable class of finitely ground programs.
\newblock In {\em Proc. ICLP'09}. Springer, 489--493.

\bibitem[\protect\citeauthoryear{Lifschitz, Pearce, and Valverde}{Lifschitz
  et~al\mbox{.}}{2001}]{lifschitz2001strongly}
{\sc Lifschitz, V.}, {\sc Pearce, D.}, {\sc and} {\sc Valverde, A.} 2001.
\newblock Strongly equivalent logic programs.
\newblock {\em ACM TOCL\/}~{\em 2,\/}~4, 526--541.

\bibitem[\protect\citeauthoryear{Lifschitz, Tang, and Turner}{Lifschitz
  et~al\mbox{.}}{1999}]{lifschitz1999nested}
{\sc Lifschitz, V.}, {\sc Tang, L.~R.}, {\sc and} {\sc Turner, H.} 1999.
\newblock Nested expressions in logic programs.
\newblock {\em Annals of Mathematics and Artificial Intelligence\/}~{\em
  25,\/}~3-4, 369--389.

\bibitem[\protect\citeauthoryear{Littman, Goldsmith, and Mundhenk}{Littman
  et~al\mbox{.}}{1998}]{littman1998computational}
{\sc Littman, M.~L.}, {\sc Goldsmith, J.}, {\sc and} {\sc Mundhenk, M.} 1998.
\newblock The computational complexity of probabilistic planning.
\newblock {\em J. Artificial Intelligence Research\/}~{\em 9}, 1--36.

\bibitem[\protect\citeauthoryear{Marek, Niemela, et~al\mbox{.}}{Marek
  et~al\mbox{.}}{2006}]{marek2006logic}
{\sc Marek, V.~W.}, {\sc Niemela, I.}, {\sc et~al\mbox{.}} 2006.
\newblock Logic programs with monotone abstract constraint atoms.
\newblock {\em arXiv preprint cs/0608103\/}.

\bibitem[\protect\citeauthoryear{Niemel{\"a}, Simons, and Soininen}{Niemel{\"a}
  et~al\mbox{.}}{1999}]{niemela1999stable}
{\sc Niemel{\"a}, I.}, {\sc Simons, P.}, {\sc and} {\sc Soininen, T.} 1999.
\newblock Stable model semantics of weight constraint rules.
\newblock In {\em Proc.\ LPNMR'99}. Springer, 317--331.

\bibitem[\protect\citeauthoryear{Pearce and Valverde}{Pearce and
  Valverde}{2008}]{10.1007/978-3-540-89982-2_46}
{\sc Pearce, D.} {\sc and} {\sc Valverde, A.} 2008.
\newblock Quantified equilibrium logic and foundations for answer set programs.
\newblock In {\em Proc.\ ICLP'08}, {M.~Garcia de~la Banda} {and} {E.~Pontelli},
  Eds. Springer, 546--560.

\bibitem[\protect\citeauthoryear{Son, Pontelli, and Tu}{Son
  et~al\mbox{.}}{2007}]{son2007answer}
{\sc Son, T.~C.}, {\sc Pontelli, E.}, {\sc and} {\sc Tu, P.~H.} 2007.
\newblock Answer sets for logic programs with arbitrary abstract constraint
  atoms.
\newblock {\em Journal of Artificial Intelligence Research\/}~{\em 29},
  353--389.

\end{thebibliography}
\bibliographystyle{acmtrans.bst}

%


\section*{Supplementary material}

### Appendix A  Proofs and additional details

In the following, we give proofs of the theorems and propositions of the paper. Furthermore, some more detailed explanations are included.

Theorems and Propositions have the same numbering as in the original document. Auxiliary Theorems etc. are numbered with roman numerals.

*Expressing unweighted formulas over semirings*  If we have a semiring $\mathscr{R} = (R, \oplus, \otimes, e_\oplus, e_\otimes)$, where the addition $\oplus$ is invertible, then we can translate any formula $\phi$ without weights into one with weights $\tau(\phi)$ that does not use the boolean connectives and such that

$$\llbracket \tau(\phi) \rrbracket^\sigma_\mathscr{R}(\mathscr{I}_w) = e_\otimes \text{ if } \mathscr{I}_w \models \phi \text{ and } \llbracket \tau(\phi) \rrbracket^\sigma_\mathscr{R}(\mathscr{I}_w) = e_\oplus \text{ if } \mathscr{I}_w \not\models \phi. \tag{A1}$$

We define $\tau(\phi)$ inductively as follows:

- if $\phi = p(\vec{x})$, then $\tau(\phi) = p(\vec{x})$
- if $\phi = \phi_1 \wedge \phi_2$, then $\tau(\phi) = \tau(\phi_1) * \tau(\phi_2)$
- if $\phi = \phi_1 \vee \phi_2$, then $\tau(\phi) = e_\otimes + -(e_\otimes + -\tau(\phi_1)) * (e_\otimes + -\tau(\phi_2))$
- if $\phi = \phi_1 \rightarrow \phi_2$, then $\tau(\phi) = \tau(\phi_1) \rightarrow \tau(\phi_2)$
- if $\phi = \forall y \phi_1(y)$, then $\tau(\phi) = \Pi y \tau(\phi_1(y))$
- if $\phi = \exists y \phi_1(y)$, then $\tau(\phi) = e_\otimes + -\Pi y(e_\otimes + -\tau(\phi_1(y)))$

We can show by induction that $\llbracket \tau(\phi) \rrbracket_\mathscr{R}(\mathscr{I}_w) \in \{e_\oplus, e_\otimes\}$ and thus that the product $\Pi y$ is always defined. Therefore, this encoding even includes the quantifiers. Furthermore, it is easy to see that the formula $\tau(\phi)$ is constructible in linear time from $\phi$, and that hence the size of $\tau(\phi)$ is linear in the size $\phi$. It is thus possible to eliminate unweighted formulas with polynomial overhead if addition $\oplus$ is invertible.

We note that while addition over the natural numbers $\mathbb{N}$ is not invertible, we can use the encoding $\tau(\phi)$ by moving from $\mathbb{N}$ to the integers $\mathbb{Z}$ and exploiting the fact that every natural number can be written as the sum of the squares of four integers (known as Lagrange's Four-Square Theorem); we then can add for global variables $x$ over $\mathbb{N}$ the algebraic constraint $x \geq_\mathbb{Z} 0$ and use for local variables $y$ in constraint formulas over $\mathbb{N}$ the expression $y_1^2 + y_2^2 + y_3^2 + y_4^2$ over four local variables $y_1, \ldots, y_4$ over $\mathbb{Z}$.

A similar translation is possible if the addition $\oplus$ is idempotent, i.e., $k \oplus k = k$ (as in the Boolean semiring), where the value of $\Sigma y \alpha(y)$ over $\mathscr{I}_w$ is naturally defined as $k$ whenever a non-empty support $\text{supp}_\oplus(\alpha(x), \mathscr{I}_w)$ leads to the single value $k$, i.e., for every $\xi \in \text{supp}_\oplus(\alpha(x), \mathscr{I}_w)$ we have $\llbracket \alpha(\xi) \rrbracket^\sigma_\mathscr{R}(\mathscr{I}_w) = k$, since $\text{supp}_\oplus(\tau(\phi), \mathscr{I}_w)$ leads to the single value $e_\otimes$.

### 3  ASP($\mathscr{AC}$)

**Proposition 6** (Generalisation). *Let $\phi$ be a $\sigma$-sentence and $\mathscr{I}_w$ be a pointed $\sigma$-HT-interpretation. Then, for the weighted $\sigma$-sentence $\alpha$ over the Boolean semiring $\mathbb{B}$, obtained from $\phi$ by replacing $\bot, \vee, \wedge, \rightarrow, \exists, \forall$ with $0, +, *, \rightarrow_\mathbb{B}, \Sigma, \Pi$, respectively, we have $\llbracket \alpha \rrbracket^\sigma_\mathbb{B}(\mathscr{I}_w) = 1$ iff $\mathscr{I}_w \models_\sigma \phi$.*

*Proof.* The claim can be easily verified by comparing the weighted semantics for $\mathscr{R} = \mathbb{B}$ and the unweighted semantics. We consider the case $\phi = \psi \rightarrow \theta$ in more detail. Then $\alpha = \beta \rightarrow_\mathbb{B} \theta$,

where $\beta$ and $\gamma$ correspond to the rewritten versions of $\psi$ and $\theta$.

$$\begin{aligned}
\mathcal{I}_w \models_\sigma \phi &\iff \mathcal{I}_{w'} \not\models_\sigma \psi \text{ or } \mathcal{I}_{w'} \models_\sigma \theta \text{ for } w' \geq w \\
&\iff [\![\beta]\!]_{\mathbb{B}}^\sigma(\mathcal{I}_w) \neq 1 \text{ or } [\![\gamma]\!]_{\mathbb{B}}^\sigma(\mathcal{I}_w) = 1 \text{ for } w' \geq w \\
&\iff [\![\beta]\!]_{\mathbb{B}}^\sigma(\mathcal{I}_w) = 0 \text{ or } [\![\gamma]\!]_{\mathbb{B}}^\sigma(\mathcal{I}_w) \neq 0 \text{ for } w' \geq w \\
&\iff [\![\alpha]\!]_{\mathbb{B}}^\sigma(\mathcal{I}_w) = 1
\end{aligned}$$

□

The proof for $\rightarrow_\mathcal{R}$ works over any semiring and not just the Boolean semiring. Therefore, we can drop the subscript $\mathcal{R}$.

**Proposition 8** (Persistence). *For any $\sigma$-sentence $\phi$ and $\sigma$-HT-interpretation $(\mathcal{I}^H, \mathcal{I}^T)$, it holds that $\mathcal{I}_H \models_\sigma \phi$ implies $\mathcal{I}_T \models_\sigma \phi$.*

*Proof.* It is know that the proposition holds for formulas $\phi$ without algebraic constraints (Pearce and Valverde 2006). We can use the same proof by structural induction, given that we can prove that the claim holds for the additional base case $\phi = k \sim_\mathcal{R} \alpha$.

In this case however, the definition of satisfaction tells us that

$$\mathcal{I}_w \models_\sigma k \sim_\mathcal{R} \alpha \iff k \sim [\![\alpha]\!]_\mathcal{R}^\sigma(\mathcal{I}_{w'}), \text{ for all } w' \geq w.$$

So, since $T \geq H$ from $\mathcal{I}_H \models_\sigma k \sim_\mathcal{R} \alpha$ follows $\mathcal{I}_T \models_\sigma k \sim_\mathcal{R} \alpha$. □

## 4 Constructs in ASP($\mathcal{AC}$) and in other formalisms

*Conditionals* We consider in more detail, how the conditional semantics *vc* and *df* of (Cabalar et al. 2020) can be modelled in our formalism. Since we do not capture arbitrary constraints as Cabalar et al. do, we assume instead that conditionals in weighted formulas are allowed and show that it is unnecessary to allow them explicitly. We start with *vc*.

Let $r(s) = H(r(s)) \leftarrow B(r(s))$ be some rule containing a conditional $s = (s'|s'' : \phi)$ which is supposed to be evaluated under *vc* semantics. This means that

$$\mathcal{I}_w \models r(s) \iff \begin{cases} \mathcal{I}_w \models r(s') & \text{if } \mathcal{I}_w \models \phi \\ \mathcal{I}_w \models r(s'') & \text{if } \mathcal{I}_w \models \neg\phi \\ \mathcal{I}_T \models r(s) & \text{otherwise.} \end{cases}$$

Now if we simply replace $s$ by $\phi * s' + \neg\phi * s''$ we get

$$\mathcal{I}_w \models r(\phi * s' + \neg\phi * s'') \iff \begin{cases} \mathcal{I}_w \models r(s') & \text{if } \mathcal{I}_w \models \phi \\ \mathcal{I}_w \models r(s'') & \text{if } \mathcal{I}_w \models \neg\phi \\ \mathcal{I}_T \models r(s) \text{ and } \mathcal{I}_w \models r(e_\oplus) & \text{otherwise.} \end{cases}$$

This is obviously different, however if we use instead the rule

$$r'(s) = H(r(s)) \leftarrow B(r(s)), 1 =_\mathbb{B} \phi + \neg\phi$$

we obtain

$$\mathscr{I}_w \models r'(\phi * s' + \neg \phi * s'')$$
$$\iff \begin{cases} \mathscr{I}_w \models r'(s') & \text{if } \mathscr{I}_w \models \phi \\ \mathscr{I}_w \models r'(s'') & \text{if } \mathscr{I}_w \models \neg\phi \\ \mathscr{I}_T \models r'(s) \text{ and } \mathscr{I}_w \models H(r(e_\oplus)) \leftarrow B(r(e_\oplus)), 1 =_\mathbb{B} \phi + \neg\phi & \text{otherwise.} \end{cases}$$
$$\iff \begin{cases} \mathscr{I}_w \models r(s') & \text{if } \mathscr{I}_w \models \phi \\ \mathscr{I}_w \models r(s'') & \text{if } \mathscr{I}_w \models \neg\phi \\ \mathscr{I}_T \models r(s) \text{ and } \mathscr{I}_w \models H(r(e_\oplus)) \leftarrow B(r(e_\oplus)), 1 =_\mathbb{B} 0 & \text{otherwise.} \end{cases}$$
$$\iff \begin{cases} \mathscr{I}_w \models r(s') & \text{if } \mathscr{I}_w \models \phi \\ \mathscr{I}_w \models r(s'') & \text{if } \mathscr{I}_w \models \neg\phi \\ \mathscr{I}_T \models r(s) & \text{otherwise.} \end{cases}$$

as desired.

In order to model *df*, we further need that the addition $\oplus$ of the semiring $\mathscr{R} = (R, \oplus, \otimes, e_\oplus, e_\otimes)$ is invertible, i.e. that we can use the connective $-$. Assume this is the case and let $s = (s' \mid s'' : \phi)$ be a conditional over the semiring $\mathscr{R}$ that we want to evaluate under *df*. Then its semantics is

$$df_{\mathscr{I}_w}(s) = \begin{cases} s', & \text{if } \mathscr{I}_w \models_\sigma \phi, \\ s'', & \text{otherwise.} \end{cases}$$

We can simply use the weighted formula $\phi * s' + (e_\otimes + -\phi) * s''$:

$$[\![\phi * s' + (e_\otimes + -\phi) * s'']\!]_\mathscr{R}(\mathscr{I}_w) = [\![\phi * s']\!]_\mathscr{R}(\mathscr{I}_w) \oplus [\![(e_\otimes + -\phi) * s'']\!]_\mathscr{R}(\mathscr{I}_w)$$
$$= \begin{cases} e_\otimes \otimes [\![s']\!]_\mathscr{R}(\mathscr{I}_w) \oplus (e_\otimes \oplus -e_\otimes) \otimes [\![s'']\!]_\mathscr{R}(\mathscr{I}_w) & \mathscr{I}_w \models \phi \\ e_\oplus \otimes [\![s']\!]_\mathscr{R}(\mathscr{I}_w) \oplus (e_\otimes \oplus -e_\oplus) \otimes [\![s'']\!]_\mathscr{R}(\mathscr{I}_w) & \text{otherwise} \end{cases}$$
$$= \begin{cases} s' \oplus e_\oplus \otimes s'' & \mathscr{I}_w \models \phi \\ e_\otimes \otimes s'' & \text{otherwise} \end{cases}$$
$$= \begin{cases} s' & \mathscr{I}_w \models \phi \\ s'' & \text{otherwise} \end{cases}$$

*Constraints in the heads of rules*

**Proposition 13** (Algebraic Choice Semantics). *Let $r = k \sim^c_\mathscr{R} \alpha \leftarrow B(r)$ be a rule with global variables $x_1, \ldots, x_n$ and $(\mathscr{I}^H, \mathscr{I}^T)$ be a $\sigma$-HT-interpretation. Then $\mathscr{I}_H \models_\sigma k \sim^c_\mathscr{R} \alpha \leftarrow B(r)$ iff*

(i) $\mathscr{I}_H \models_\sigma k \sim_\mathscr{R} \alpha \leftarrow B(r)$ *and*
(ii) *for all $\xi_1 \in r(s(x_1)), \ldots, \xi_n \in r(s(x_n))$ it holds that if $\mathscr{I}_H \models_\sigma (B_\wedge(r)(\xi_1, \ldots, \xi_n))^\sigma$, then*
$$[\![(\alpha(\xi_1, \ldots, \xi_n))^\Sigma]\!]^\sigma_\mathscr{R}(\mathscr{I}_T) = [\![(\alpha(\xi_1, \ldots, \xi_n))^\Sigma]\!]^\sigma_\mathscr{R}(\mathscr{I}_H)$$

*Proof.* According to our definition

$$\mathscr{I}_H \models_\sigma k \sim^c_\mathscr{R} \alpha \leftarrow B(r)$$
$$\iff \mathscr{I}_H \models_\sigma k \sim_\mathscr{R} \alpha \leftarrow B(r) \text{ and } \mathscr{I}_H \models_\sigma X =_\mathscr{R} \alpha \leftarrow X =_\mathscr{R} \alpha^{\neg\neg}, B(r)$$
$$\iff \mathscr{I}_H \models_\sigma k \sim_\mathscr{R} \alpha \leftarrow B(r) \text{ and } \mathscr{I}_H \models_\sigma \forall x_1, \ldots, x_n \forall X \, X =_\mathscr{R} \alpha^{\neg\neg}, B(r) \rightarrow X =_\mathscr{R} \alpha$$

By the definition of $\alpha^{\neg\neg}$ we know that $[\![\alpha^{\neg\neg}]\!]_{\mathcal{R}}(\mathscr{I}_H)$ is equal to $[\![\alpha]\!]_{\mathcal{R}}(\mathscr{I}_T)$. Therefore, we only need to consider the grounding of the rule where $X$ is replaced by $[\![\alpha]\!]_{\mathcal{R}}(\mathscr{I}_T)$. Then

$$\mathscr{I}_H \models r$$
$$\iff \mathscr{I}_H \models_\sigma k \sim_{\mathcal{R}} \alpha \leftarrow B(r) \text{ and } \mathscr{I}_H \models_\sigma \forall x_1,\ldots,x_n\, B(r) \rightarrow [\![\alpha]\!]_{\mathcal{R}}(\mathscr{I}_T) =_{\mathcal{R}} \alpha$$

□

## 5 Provenance

For provenance we define a translation of a positive datalog program $\Pi = \{r_1,\ldots,r_m\}$ as follows. First, we discuss terminology. For each predicate $q$ in $\Pi$ we introduce the following predicates.

- $p_q(\vec{Y},V,L,i,\vec{Z})$, which stores the value $V$ of the provenance of $q(\vec{Y})$ using any derivation that uses exactly $L$ leaf nodes, uses the rule $r_i$ last and the global variables in $r_i$ that do not occur in the head of $r_i$ had the value $\vec{Z}$.
- $p_q(\vec{Y},V,L,i)$, which stores the value $V$ of the provenance of $q(\vec{Y})$ using any derivation that uses exactly $L$ leaf nodes, uses the rule $r_i$ last and the global variables in $r_i$ that do not occur in the head of $r_i$ took any value.
- $p_q(\vec{Y},V,L)$, which stores the value $V$ of the provenance of $q(\vec{Y})$ using any derivation that uses exactly $L$ leaf nodes and uses any rule $r_i$ last.
- $p_q(\vec{Y},V)$, which stores the value $V$ of the provenance of $q(\vec{Y})$.
- $d_q(\vec{Y},L,i,\vec{Z})$, which asserts that there is a derivation of $q(\vec{Y})$ using exactly $L$ leaf nodes that uses the rule $r_i$ last and the global variables in $r_i$ that do not occur in the head of $r_i$ had the value $\vec{Z}$.
- $d_q(\vec{Y},L,i)$, which asserts that there is a derivation of $q(\vec{Y})$ using exactly $L$ leaf nodes that uses the rule $r_i$ last and the global variables in $r_i$ that do not occur in the head of $r_i$ took any value.
- $d_q(\vec{Y},L)$, which asserts that there is a derivation of $q(\vec{Y})$ using exactly $L$ leaf nodes that uses any rule $r_i$.

Let

$$r_i = r(\vec{Y}) \leftarrow q_1(\vec{X}_1),\ldots,q_n(\vec{X}_n)$$

be some rule with index $i$, where w.l.o.g. $n > 1$ (we can always add a new extensional atom $e()$ with provenance $e_\otimes$. We add the following rules to our translation $T(\Pi)$:

$$p_r(\vec{Y},V,L,i,\vec{Z}) \leftarrow p_{q_1}(\vec{X}_1,V_1,L_1),\ldots,p_{q_n}(\vec{X}_n,V_n,L_n), L =_{\mathbb{N}} L_1+\ldots+L_n, V =_{\mathcal{R}} V_1 * \ldots * V_n \tag{E1}$$

$$d_r(\vec{Y},L,i,\vec{Z}) \leftarrow p_{q_1}(\vec{X}_1,V_1,L_1),\ldots,p_{q_n}(\vec{X}_n,V_n,L_n), L =_{\mathbb{N}} L_1+\ldots+L_n \tag{E2}$$

$$p_r(\vec{Y},V,L,i) \leftarrow d_r(\vec{Y},L,i,\vec{Z}), V =_{\mathcal{R}} p_r(\vec{Y},V^*,L,i,\vec{Z}^*) * V^* \tag{E3}$$

$$d_r(\vec{Y},L,i) \leftarrow d_r(\vec{Y},L,i,\vec{Z}) \tag{E4}$$

Here $\vec{Z} = \bigcup_{i=1}^n \vec{X}_i \setminus \vec{Y}$, i.e. the global variables of $r_i$ that do not occur in its head.

Further, for every predicate $q$ in $\Pi$ we add the rules

$$p_r(\vec{Y},V,L) \leftarrow d_r(\vec{Y},L,I), V =_{\mathscr{R}} p_r(\vec{Y},V^*,L,I^*) * V^* \tag{E5}$$

$$d_r(\vec{Y},L) \leftarrow d_r(\vec{Y},L,I) \tag{E6}$$

$$p_r(\vec{Y},V) \leftarrow d_r(\vec{Y},L), V =_{\mathscr{R}} p_r(\vec{Y},V^*,L^*) * V^* \tag{E7}$$

This translation works as follows. The last rule sums up the value of each derivation tree to obtain the final provenance, whereas the previous rules calculate the provenance of less and less restricted derivation trees. Therefore, there is at least one answer set $\mathscr{I}$ such that $p_r(\vec{Y},V) \in \mathscr{I}$ iff the provenance of $r(\vec{Y})$ is $V$. On the other hand, since these rules are also all positive, there is exactly one answer set. The following results shows the correctness of the translation.

**Theorem 14** (Provenance Encoding). *Given a positive datalog program $\Pi$ the $\mathscr{AC}$-program $T(\Pi)$ computes the provenance semantics over the $\omega$-continuous semiring $\mathscr{R}$ in the following sense. Let $D$ be an edb and $r(\vec{x})$ a query result of $D \cup \Pi$ with semiring provenance $v$. Then the unique equilibrium model $\mathscr{I}$ of $T(\Pi) \cup \{p_e(\vec{x},v,1) \leftarrow \mid (e(\vec{x}),v) \in D\}$ contains $p_r(\vec{x},v')$ iff $v' = v$.*

*Proof.* Let $\mathscr{I}$ be the unique equilibrium model of $T(\Pi) \cup \{p_e(\vec{x},v,1) \leftarrow \mid (e(\vec{x}),v) \in D\}$. We proceed by induction on the number $L$ of leaf nodes that are used in the derivation tree, to show that our construction is correct and the predicates indeed behave as they should according to their description, which among other things implies that $v$ is the provenance of the query result $r(\vec{x})$ iff the unique equilibrium model $\mathscr{I}$ of $T(\Pi) \cup \{p_e(\vec{x},v,1) \leftarrow \mid (e(\vec{x}),v) \in D\}$ contains $p_r(\vec{x},v)$. In the proof we only consider the predicates for the values of the provenance in detail. The correctness of the derivability predicates $d_r(\vec{Y},L[,i,\vec{Z}])$ follows analogous reasoning since the rules for derivability are just simplified versions of the rules for the provenance values.

The case $L = 0$ is impossible, since we always use at least one leaf node in each derivation.

The case $L = 1$ occurs exactly when $r(\vec{Y})$ is a leaf node. Since edb predicates do not occur in heads of rules in $\Pi$, the only rules we have to consider are of the form $p_e(\vec{x},v,1) \leftarrow$. Here, the claim holds.

Assume the claim holds for all $L' < L$.

Consider the rule (E1). For the body to be satisfied, we need that $p_{q_1}(\vec{X}_1,V_1,L_1),\ldots,p_{q_n}(\vec{X}_n,V_n,L_n)$ are contained in $\mathscr{I}$. Since $L_i > 0$, $L = L_1 + \cdots + L_n$ and $n > 1$, we know that $L_i < L$ and therefore the claim holds for $p_{q_1}(\vec{X}_1,V_1,L_1),\ldots,p_{q_n}(\vec{X}_n,V_n,L_n)$. Therefore, $p_{q_i}(\vec{x}_i,v_i,l_i) \in \mathscr{I}$ iff $v_i$ is the provenance of $q_i(\vec{x}_i)$ using any derivation that uses exactly $l_i$ leaf nodes. Let us denote by $dtree_\ell(q_i(\vec{x}_i))$ the set of all derivation trees $\tau$ for $q_i(\vec{x}_i)$ that use exactly $\ell$ leaf nodes, and by $leaves(\tau)$ the set of leaf nodes in the tree $\tau$. Then

$$v_i = \bigoplus_{\tau \in dtree_{l_i}(q_i(\vec{x}_i))} \bigotimes_{t \in leaves(\tau)} R(t). \tag{E8}$$

Then for $v = v_1 \otimes \cdots \otimes v_n$, and $l = l_1 + \cdots + l_n$ we have that $p_r(\vec{y},v,l,i,\vec{z})$ is in $\mathscr{I}$ iff for $i = 1,\ldots,n$ the atoms $p_{q_i}(\vec{x}_i,v_i,l_i)$ are in $\mathscr{I}$. This however means that for $i = 1,\ldots,n$ the equation (E8) holds. Therefore $v$ is the value

$$v = v_1 \otimes \ldots \otimes v_n = \bigoplus_{\tau \in dtree_{l_1}(q_1(\vec{x}_1))} \bigotimes_{t \in leaves(\tau)} R(t) \otimes \cdots \otimes \bigoplus_{\tau \in dtree_{l_n}(q_n(\vec{x}_n))} \bigotimes_{t \in leaves(\tau)} R(t).$$

We use that for every combination of derivations $\tau_1,\ldots,\tau_n$ respectively for $q_1(\vec{x}_1),\ldots,q_n(\vec{x}_n)$ there is a derivation of $r(\vec{y})$ using the rule $r_i$ last, where the global variables that do not occur in

the head of $r_i$ have the value $\vec{z}$. According to the distributive law, assuming that $last(\tau)$ denotes the last rule in derivation tree $\tau$ and that $gvar(r_i)$ denotes the value of the global variables in rule $r_i$ that do not occur in the head of $r_i$, we obtain that

$$v = v_1 \otimes \ldots \otimes v_n = \bigoplus_{\tau \in dtree_l(r(\vec{y})), r_i = last(\tau), gvar(r_i) = \vec{z}} \bigotimes_{t \in leaves(\tau)} R(t).$$

It follows that rule (E1) ensures that the predicates of the form $p_r(\vec{Y}, V, L, i, \vec{Z})$ satisfy our claim (for $L$).

Next for rule (E3). Here, we simply aggregate over the global variables in $r_i$ that do not occur in the head of $r_i$.

The head atom $p_r(\vec{Y}, V, L, i)$ should describe the value $V$ of the provenance of $r(\vec{Y})$ using any derivation that uses exactly $L$ leaf nodes and uses rule $r_i$ last. This value is given by

$$\bigoplus_{\tau \in dtree_L(r(\vec{x})), r_i = last(\tau)} \bigotimes_{t \in leaves(\tau)} R(t) = \bigoplus_{\vec{z}} \bigoplus_{\tau \in dtree_l(r(\vec{x})), r_i = last(\tau), gvar(r_i) = \vec{z}} \bigotimes_{t \in leaves(\tau)} R(t).$$

We know that according to the previous rule (E1), the predicate $p_r(\vec{Y}, V, L, i, \vec{Z})$ encodes exactly the inner sum. Since rule (E3) performs the outer sum, it follows that rule (E3) ensures that the predicates of the form $p_r(\vec{Y}, V, L, i)$ satisfy our claim (for $L$).

Next, the rule (E5) aggregates over the different rules that were used last to derive $r(\vec{x})$ using $l$ leaf nodes. The argumentation is analogous to the one for the last rule (E3), where we aggregated over $\vec{z}$ instead of the rule index $i$ like here.

Overall, the inductive proof of the correctness of the rules specifying the predicates $p_r(\vec{Y}, V, L[, i, \vec{Z}])$ succeeds, since all the predicates are correctly defined for $L$ given that predicates are well defined for $L' < L$.

Last but not least we consider the last rule (E7) which should produce the final result. According to the definition of provenance for datalog programs in (Green et al. 2007), the label $v$ of the query result $r(\vec{x})$ is

$$v = \bigoplus_{\tau \in dtree(r(\vec{x}))} \bigotimes_{t \in leaves(\tau)} R(t),$$

where $dtree(r(\vec{x}))$ is the set of all derivation trees for $r(\vec{x})$ and $R(t)$ is the provenance of the leaf $t$. We reformulate this equation as follows:

$$v = \bigoplus_{\tau \in dtree(r(\vec{x}))} \bigotimes_{t \in leaves(\tau)} R(t) = \bigoplus_{l \geq 0} \bigoplus_{\tau \in dtree_l(r(\vec{x}))} \bigotimes_{t \in leaves(\tau)} R(t);$$

since $p_r(\vec{x}, v', l) \in \mathscr{I}$ iff $v' = \bigoplus_{\tau \in dtree_l(r(\vec{x}))} \bigotimes_{t \in leaves(\tau)} R(t)$, we obtain

$$v = \bigoplus_{l \geq 0} \bigoplus_{p_r(\vec{x}, v', l) \in \mathscr{I}} v'.$$

As due to rule (E7), we have $p_r(\vec{x}, v) \in \mathscr{I}$ iff

$$v = \bigoplus_{l \geq 0} \bigoplus_{p_r(\vec{x}, v', l) \in \mathscr{I}} v',$$

we obtain that $p_r(\vec{x}, v) \in \mathscr{I}$. This concludes the proof. □

## 6 Language Aspects

### 6.1 Safety

**Theorem I** (Support Independence). *Let $\sigma_1, \sigma_2$ be semiring signatures, $\mathscr{I}_w$ be a pointed $\sigma_i$-HT-interpretation ($i = 1, 2$) and $\alpha$ be a weighted $\sigma_i$-formula ($i = 1, 2$) that is syntactically domain independent w.r.t. variable x. Then*

$$\{\xi \in r_1(s(x)) \mid [\![\alpha(\xi)]\!]_{\mathscr{R}}^{\sigma_1}(\mathscr{I}_w) \neq e_\oplus\} = \{\xi \in r_2(s(x)) \mid [\![\alpha(\xi)]\!]_{\mathscr{R}}^{\sigma_2}(\mathscr{I}_w) \neq e_\oplus\}.$$

*Proof.* We give a proof using structural induction on the syntactically domain independent formula $\alpha$. In the following let $\mathscr{I}_w$ some pointed HT-interpretation and $\sigma_1, \sigma_2$ semiring signatures that contain all the constants of $\alpha$ and $\mathscr{I}_w$

- Case $\alpha = k$:
  This formula contains no local variables, therefore the equality is trivially fulfilled.
- Case $\alpha = \phi(\vec{x})$:
  The given formulas are all range restricted. For range restricted formulas it is known, that they are domain independent (see for example (Demolombe 1992)), which implies that when they are seen as weighted formulas, their support does not depend on the signature.
- Case $\alpha = \neg_\oplus \beta(x)$:
  The semantics of $\neg_\oplus \beta$ is the inverse of the semantics of $\beta$ w.r.t. $\oplus$, which is $e_\oplus$ iff the semantics of $\beta$ is $e_\oplus$. Therefore we have

$$\{\xi \in r_1(s(x)) \mid [\![\neg_\oplus \beta(\xi)]\!]_{\mathscr{R}}^{\sigma_1}(\mathscr{I}_w) \neq e_\oplus\}$$
$$= \{\xi \in r_1(s(x)) \mid [\![\beta(\xi)]\!]_{\mathscr{R}}^{\sigma_1}(\mathscr{I}_w) \neq e_\oplus\}$$
$$= \{\xi \in r_2(s(x)) \mid [\![\beta(\xi)]\!]_{\mathscr{R}}^{\sigma_2}(\mathscr{I}_w) \neq e_\oplus\}$$
$$= \{\xi \in r_2(s(x)) \mid [\![\neg_\oplus \beta(\xi)]\!]_{\mathscr{R}}^{\sigma_2}(\mathscr{I}_w) \neq e_\oplus\}.$$

- Case $\alpha = \neg_\otimes \beta(x)$:
  The semantics of $\neg_\otimes \beta$ is the inverse of the semantics of $\beta$ w.r.t. $\otimes$ or $e_\oplus$ if the semantics of $\beta$ is $e_\oplus$. Therefore we have on the one hand that

$$[\![\beta(\xi)]\!]_{\mathscr{R}}^{\sigma_i}(\mathscr{I}) = e_\oplus \Rightarrow [\![\neg_\otimes \beta(\xi)]\!]_{\mathscr{R}}^{\sigma_i}(\mathscr{I}) = e_\oplus.$$

Furthermore, we have for the other direction that

$$[\![\neg_\otimes \beta(\xi)]\!]_{\mathscr{R}}^{\sigma_i}(\mathscr{I}) = e_\oplus \Rightarrow [\![\beta(\xi)]\!]_{\mathscr{R}}^{\sigma_i}(\mathscr{I}) = e_\oplus \vee [\![\beta(\xi)]\!]_{\mathscr{R}}^{\sigma_i}(\mathscr{I}) \otimes e_\oplus = e_\otimes$$

The second disjunct implies that $e_\oplus = e_\otimes$ since $e_\oplus$ annihilates $R$. Therefore since $\forall r \in R : e_\otimes \otimes r = r$ holds we have that $\forall r \in R : e_\oplus \otimes r = e_\oplus = r$, meaning our semiring has exactly one element, namely $e_\oplus$. Therefore we have

$$[\![\neg_\otimes \beta(\xi)]\!]_{\mathscr{R}}^{\sigma_i}(\mathscr{I}) = e_\oplus \Rightarrow [\![\beta(\xi)]\!]_{\mathscr{R}}^{\sigma_i}(\mathscr{I}) = e_\oplus$$

So we know that the semantics of $\neg_\otimes \beta$ is $e_\oplus$ iff the semantics of $\beta$ is $e_\oplus$ and as in the previous case we obtain

$$\{\xi \in r_1(s(x)) \mid [\![\neg_\otimes \beta(\xi)]\!]_{\mathscr{R}}^{\sigma_1}(\mathscr{I}_w) \neq e_\oplus\}$$
$$= \{\xi \in r_1(s(x)) \mid [\![\beta(\xi)]\!]_{\mathscr{R}}^{\sigma_1}(\mathscr{I}_w) \neq e_\oplus\}$$
$$= \{\xi \in r_2(s(x)) \mid [\![\beta(\xi)]\!]_{\mathscr{R}}^{\sigma_2}(\mathscr{I}_w) \neq e_\oplus\}$$
$$= \{\xi \in r_2(s(x)) \mid [\![\neg_\otimes \beta(\xi)]\!]_{\mathscr{R}}^{\sigma_2}(\mathscr{I}_w) \neq e_\oplus\}.$$

- Case $\alpha = \neg\neg\alpha_1(x)$:
  We know that $\{\xi \in r_1(s(x)) \mid [\![\neg\neg\alpha_1(\xi)]\!]_{\mathcal{R}}^{\sigma_1}(\mathcal{I}_w) \neq e_\oplus\}$ is equal to $\{\xi \in r_1(s(x)) \mid [\![\alpha_1(\xi)]\!]_{\mathcal{R}}^{\sigma_1}(\mathcal{I}_w) \neq e_\oplus\}$. Therefore this case follows immediately from the inductive hypothesis for $\alpha_1(x)$.

- Case $\alpha = \alpha_1(x) + \alpha_2(x)$:
  Assume that there is
  $$\xi \in r_2(s(x)) \setminus r_1(s(x)) \text{ s.t. } [\![\alpha_1(\xi) + \alpha_2(\xi)]\!]_{\mathcal{R}}^{\sigma_2}(\mathcal{I}_w) \neq e_\oplus\},$$
  then we know that
  $$[\![\alpha_1(\xi)]\!]_{\mathcal{R}}(\mathcal{I}) \neq e_\oplus \text{ or } [\![\alpha_2(\xi)]\!]_{\mathcal{R}}(\mathcal{I}) \neq e_\oplus$$
  and further that $\xi \notin \mathcal{D}_2$, since $\sigma_1, \sigma_2$ are semiring signatures. Then it however holds that
  $$\xi \in \{\xi \in r_2(s(x)) \mid [\![\alpha_1(\xi)]\!]_{\mathcal{R}}^{\sigma_2}(\mathcal{I}_w) \neq e_\oplus\} \setminus \{\xi \in r_1(s(x)) \mid [\![\alpha_1(\xi)]\!]_{\mathcal{R}}^{\sigma_1}(\mathcal{I}_w) \neq e_\oplus\}$$
  or
  $$\xi \in \{\xi \in r_2(s(x)) \mid [\![\alpha_2(\xi)]\!]_{\mathcal{R}}^{\sigma_2}(\mathcal{I}_w) \neq e_\oplus\} \setminus \{\xi \in r_1(s(x)) \mid [\![\alpha_2(\xi)]\!]_{\mathcal{R}}^{\sigma_1}(\mathcal{I}_w) \neq e_\oplus\}.$$
  This is impossible due to the induction hypothesis for $\alpha_1(x)$ and $\alpha_2(x)$. Analogously we can show that
  $$\{\xi \in r_1(s(x)) \mid [\![\alpha_1(\xi) + \alpha_2(\xi)]\!]_{\mathcal{R}}^{\sigma_1}(\mathcal{I}_w) \neq e_\oplus\} \setminus \{\xi \in r_2(s(x)) \mid [\![\alpha_1(\xi) + \alpha_2(\xi)]\!]_{\mathcal{R}}^{\sigma_2}(\mathcal{I}_w) \neq e_\oplus\}$$
  is empty, and therefore that
  $$\{\xi \in r_2(s(x)) \mid [\![\alpha_1(\xi) + \alpha_2(\xi)]\!]_{\mathcal{R}}^{\sigma_2}(\mathcal{I}_w) \neq e_\oplus\}$$
  $$= \{\xi \in r_1(s(x)) \mid [\![\alpha_1(\xi) + \alpha_2(\xi)]\!]_{\mathcal{R}}^{\sigma_1}(\mathcal{I}_w) \neq e_\oplus\}.$$

- Case $\alpha = \alpha_1(x) * \alpha_2$:
  First note that:
  $$\{\xi \in r_1(s(x)) \mid [\![\alpha_1(\xi) * \alpha_2]\!]_{\mathcal{R}}^{\sigma_1}(\mathcal{I}_w) \neq e_\oplus\} \subseteq \{\xi \in r_1(s(x)) \mid [\![\alpha_1(\xi)]\!]_{\mathcal{R}}^{\sigma_1}(\mathcal{I}_w) \neq e_\oplus\}$$
  We use the induction hypothesis on $\alpha_1(x)$ to obtain the equality:
  $$\{\xi \in r_1(s(x)) \mid [\![\alpha_1(\xi) * \alpha_2]\!]_{\mathcal{R}}^{\sigma_1}(\mathcal{I}_w) \neq e_\oplus\}$$
  $$= \{\xi \in \{\xi \in r_1(s(x)) \mid [\![\alpha_1(\xi)]\!]_{\mathcal{R}}^{\sigma_1}(\mathcal{I}_w) \neq e_\oplus\} \mid [\![\alpha_1(\xi) * \alpha_2]\!]_{\mathcal{R}}^{\sigma_1}(\mathcal{I}_w) \neq e_\oplus\}$$
  $$= \{\xi \in \{\xi \in r_2(s(x)) \mid [\![\alpha_1(\xi)]\!]_{\mathcal{R}}^{\sigma_2}(\mathcal{I}_w) \neq e_\oplus\} \mid [\![\alpha_1(\xi) * \alpha_2]\!]_{\mathcal{R}}^{\sigma_2}(\mathcal{I}_w) \neq e_\oplus\}$$
  $$= \{\xi \in r_2(s(x)) \mid [\![\alpha_1(\xi) * \alpha_2]\!]_{\mathcal{R}}^{\sigma_2}(\mathcal{I}_w) \neq e_\oplus\}$$

- Case $\alpha = \alpha_1 * \alpha_2(x)$:
  works analogously to the one above.

- Case $\alpha = \alpha_1(x) * \phi(\vec{X}')$, where $\vec{X}' \subseteq \{x\}$:
  works analogously to the one above.

The proof for more than one local variable works analogously. $\square$

Theorem 16 is a corollary of Theorem I:

**Theorem 16** (Domain Independence). *If a formula is syntactically domain independent, then it is also domain independent.*

*Proof.* When we evaluate $\alpha(\vec{X})$ we take the sum over $\text{supp}_\oplus(\alpha(\vec{X}), \mathscr{I}_w)$. Due to the previous lemma we know that the support is invariant under changing the domain. Further, we know that for a given assignment of the local variables the semantics is independent of the domain. Therefore, the semantics is invariant under changing the domain for syntactically domain independent formulas. □

**Theorem 18** (Program Domain Independence). *If a program $\Pi$ is safe, then it is domain independent.*

*Proof (sketch).* Let $\sigma_i = \langle \mathscr{D}_i, \mathscr{P}, \mathscr{X}, \mathscr{S}, r_i \rangle, i = 1, 2$ be semiring signatures s.t. $\Pi$ is a $\sigma_i$-formula for $i = 1, 2$ and let $\mathscr{I}_w = (\mathscr{I}^H, \mathscr{I}^H, w)$ be a pointed $\sigma_i$-HT-interpretation for $i = 1, 2$.

Let $r \in \Pi$. If $r$ does not contain global variables, the claim is evident. Otherwise assume $r = \forall x_1, \ldots, x_n\ \alpha(x_1, \ldots, x_n)$. For $\xi_i \in r_1(s(x_i)) \cap r_2(s(x_i))$ the semantics of $\alpha(\xi_1, \ldots, \xi_n)$ does not depend on $\sigma_i$. Otherwise, assume for some $j$ it holds that $\xi_j \in r_1(s(x_j)) \setminus r_2(s(x_j))$. Then $\mathscr{I}_w \models_{\sigma_i} \alpha(\xi_1, \ldots, \xi_n)$ for $i = 1, 2$ since $x_j$ satisfies condition (ii.1) and therefore there exists an atom in the body that is not satisfied by $\mathscr{I}_w$. □

### 6.2 Program Equivalence

**Theorem 20.** *For any $\Pi_1, \Pi_2$ programs, we have that $\Pi_1 \equiv_s \Pi_2$ iff $\Pi_1$ has the same HT-models as $\Pi_2$.*

*Proof (sketch).* The direction $\Leftarrow$ is clear. For $\Rightarrow$ we can generalise the proof in (Lifschitz et al. 2001), by constructing $\Pi'$, which asserts a subset of the interpretation $\mathscr{I}^T$ that is ensured to be stable ($\mathscr{I}^H$), and a subset that if partly present is ensured to be fully present ($\mathscr{I}^T \setminus \mathscr{I}^H$).

Let $\Pi_1$ and $\Pi_2$ have different HT-models. W.l.o.g. there must be at least one HT-interpretation $(\mathscr{I}^H, \mathscr{I}^T)$ that is an HT-model of $\Pi_1$ but not of $\Pi_2$. As in (Lifschitz et al. 2001) we simply define

$$\Pi' = \{p(\vec{x}) \leftarrow \mid p(\vec{x}) \in \mathscr{I}^H\} \cup \{p(\vec{x}) \leftarrow q(\vec{y}) \mid p(\vec{x}), q(\vec{y}) \in \mathscr{I}^T \setminus \mathscr{I}^H\}$$

Then $\mathscr{I}^T$ is an equilibrium model of $\Pi_2 \cup \Pi'$, but not of $\Pi_1 \cup \Pi'$ and therefore $\Pi_1$ and $\Pi_2$ are not strongly equivalent. □

The above proof relies on the fact that our semantics is defined for program with infinite sets of rules. If we want to avoid this, there are multiple options. In (Lifschitz et al. 2007) the strong equivalence of arbitrary first-order formulas was considered and characterised as equivalence in HT Logic. The proof however uses the fact that the strong equivalence considered in their work is for any first-order sentence and not only for programs, which are a syntactic fragment. A straight forward way to reproduce their proof strategy in our setting seems not to be apparent.

Nevertheless, it is possible to prove the statement when programs are finite sets of rules in our setting, provided that auxiliary predicate symbols are available not occurring in the program (which trivially holds if we have infinitely many predicates of each arity in the underlying predicate signature $\mathscr{P}$).

*Proof (sketch).* We have two directions to prove. Based on the idea of (Lifschitz et al. 2001).
($\Rightarrow$) We prove this direction using contraposition, that is we assume that we have two programs $\Pi_1, \Pi_2$ s.t. for some HT-interpretation $(\mathscr{I}^H, \mathscr{I}^T)$ it holds that (w.l.o.g.) $(\mathscr{I}^H, \mathscr{I}^T, H) \models \Pi_1$ and $(\mathscr{I}^H, \mathscr{I}^T, H) \not\models \Pi_2$. Next we show that there exists a program $\Delta$ s.t. $\Pi_1 \cup \Delta$ and $\Pi_2 \cup \Delta$ have different answer sets.

The program $\Delta$ consists of the following rules, where $\mathscr{G}$ is the set of predicates that occur in $\Pi_1 \cup \Pi_2$:

$$\text{repair}_p(X_1,\ldots,X_n) \leftarrow \top =_\mathbb{B} \neg\neg\text{repair}_p(X_1,\ldots,X_n), \tag{F1}$$

$$\text{for } p \in \mathscr{G} \text{ with arity } n.$$

$$p(X_1,\ldots,X_n) \leftarrow \text{repair}_p(X_1,\ldots,X_n), \tag{F2}$$

$$\text{for } p \in \mathscr{G} \text{ with arity } n.$$

$$\text{fill}_{p,q}(X_1,\ldots,X_n,Y_1,\ldots,Y_m) \leftarrow \top =_\mathbb{B} \neg\neg\text{fill}_{p,q}(X_1,\ldots,X_n,Y_1,\ldots,Y_m), \tag{F3}$$

$$\text{for p, q} \in \mathscr{G} \text{ with arities } n,m.$$

$$p(X_1,\ldots,X_n) \leftarrow q(Y_1,\ldots,Y_m), \text{fill}_{p,q}(X_1,\ldots,X_n,Y_1,\ldots,Y_m), \tag{F4}$$

$$\text{for p, q} \in \mathscr{G} \text{ with arities } n,m.$$

Intuitively $\text{repair}_p(X_1,\ldots,X_n)$ guesses some tuple $(X_1,\ldots,X_n)$ for predicate $p$ such that the atom $p(X_1,\ldots,X_n)$ should definitely be satisfied. Similarly, $\text{fill}_{p,q}(X_1,\ldots,X_n,Y_1,\ldots,Y_m)$ guesses some values $(X_1,\ldots,X_n)$ and $(Y_1,\ldots,Y_m)$ for the predicates p and q, respectively, such that if $p(X_1,\ldots,X_n)$ is satisfied then also $q(Y_1,\ldots,Y_n)$ should be satisfied.

Consider now the interpretation

$$\mathscr{I}^* = \mathscr{I}^T \cup \{\text{repair}_p(x_1,\ldots,x_n) \mid p(x_1,\ldots,x_n) \in \mathscr{I}^H\} \tag{F5}$$

$$\cup \{\text{fill}_{p,q}(x_1,\ldots,x_n,y_1,\ldots,y_m) \mid p(x_1,\ldots,x_n),q(y_1,\ldots,y_n) \in \mathscr{I}^T \setminus \mathscr{I}^H\}. \tag{F6}$$

Then we have that $(\mathscr{I}^H,\mathscr{I}^*,H) \models \Pi_1 \cup \Delta$, therefore $\mathscr{I}^*$ is not an equilibrium model of $\Pi_1 \cup \Delta$. However for $\Pi_2$ we have that $(\mathscr{I}^*,\mathscr{I}^*,H) \models \Pi_2 \cup \Delta$. Furthermore, consider now some interpretation $\mathscr{I}' \subseteq \mathscr{I}^*$ s.t. $(\mathscr{I}',\mathscr{I}^*,H) \models \Pi_2 \cup \Delta$. Due to the included repairs, we know that at least $\mathscr{I}^H \subseteq \mathscr{I}'$. Moreover, we know that this inclusion is strict even when we consider only the predicates occurring in $\Pi_1 \cup \Pi_2$, since $(\mathscr{I}^H,\mathscr{I}^T,H) \not\models \Pi_2$. Therefore, due to the fills we have to include all the predicates from $\mathscr{I}^T$. It follows that $\mathscr{I}' = \mathscr{I}^*$ and therefore that $\mathscr{I}^*$ is an equilibrium model of $\Pi_2 \cup \Delta$.

($\Leftarrow$) Assume that $\Pi_1$ has the same HT-models as $\Pi_2$ and consider for an arbitrary program $\Delta$ the HT-models of $\Pi_1 \cup \Delta$ and $\Pi_2 \cup \Delta$. Those are exactly the HT-models $(\mathscr{I}^H,\mathscr{I}^T)$ s.t.

$$(\mathscr{I}^H,\mathscr{I}^T,H) \models \Pi_1 \text{ and } (\mathscr{I}^H,\mathscr{I}^T,H) \models \Delta,$$

which is however equivalent to

$$(\mathscr{I}^H,\mathscr{I}^T,H) \models \Pi_2 \text{ and } (\mathscr{I}^H,\mathscr{I}^T,H) \models \Delta$$

since $\Pi_1$ and $\Pi_2$ have the same HT-models. Since the HT-models of $\Pi_1 \cup \Delta$ and $\Pi_2 \cup \Delta$ are the same, we also know that the equilibrium models $\mathscr{I}$ are the same, since it follows that

$$(\mathscr{I},\mathscr{I},H) \models \Pi_1 \cup \Delta \text{ and } \forall \mathscr{I}' \subsetneq \mathscr{I} : (\mathscr{I}',\mathscr{I},H) \not\models \Pi_1 \cup \Delta$$

$$\iff (\mathscr{I},\mathscr{I},H) \models \Pi_2 \cup \Delta \text{ and } \forall \mathscr{I}' \subsetneq \mathscr{I} : (\mathscr{I}',\mathscr{I},H) \not\models \Pi_2 \cup \Delta.$$

$\square$

## 7 Complexity

**Theorem II** (Complexity of evaluation). *Let $\mathscr{R} = (R,\oplus,\otimes,e_\oplus,e_\otimes)$ some semiring and $e: R \to \mathbb{N}$ some encoding function s.t. $\mathscr{R}$ is efficiently encoded by e.*

*Then for a quantifier-free weighted formula over $\mathscr{R}$ and pointed HT-interpretation $\mathscr{I}_w$, we can calculate $e([\![\alpha]\!]_{\mathscr{R}}(\mathscr{I}_w))$ in polynomial time.*

*Proof.* The proof is by structural induction on the formula $\alpha$, with induction invariant that $t(\alpha)$ the time needed is in $\mathcal{O}(N^n)$, where $N$ is the size of the input, $n \in \mathbb{N}$ is a constant not depending on the input. Further, $s(\alpha)$ the size of the representation of the obtained value, i.e. $\|[\![\alpha]\!]_{\mathscr{R}}(\mathscr{I}_w)\|$, is in $\mathcal{O}(N)$.'

- Base Cases:
  — $\alpha = e(k)$: Then one can evaluate the expression by simply returning $e(k)$. This is feasible in polynomial time. The size of the output is linear in the size of the input.
  — $\alpha = \phi$: We simply check if $\mathscr{I}_w \models \phi$ and return $e_\oplus$ or $e_\otimes$ accordingly. This is possible in polynomial time since $\phi$ is quantifier-free and the size of the output is also bounded by a constant.

  We have shown the invariant for all formulae up to a certain structural complexity.

- Induction Step:
  — $\alpha = \beta_1 \to \beta_2$: We know that for $\beta_i$ the invariant holds, therefore we can check in time bounded polynomially in the size of the formula, whether $[\![\beta_i]\!]_{\mathscr{R}}(\mathscr{I}_w) = e_\oplus$ and output $e_\oplus$ or $e_\otimes$ accordingly. The size of the output is again bounded by a constant.
  — $\alpha = \beta_1 + \beta_2$: We know that the invariant holds for $\beta_1, \beta_2$. Further $t(\alpha) = t(\beta_1) + t(\beta_2) + x$, where $x$ is the time needed for addition of the results for $\beta_1$ and $\beta_2$. We know that $x$ is polynomial in $s(\beta_1) + s(\beta_2)$, which we know to be in $\mathcal{O}(N)$. Therefore $x \in \mathcal{O}(N^l)$, where $l$ is the degree of the polynomial bounding the time needed to add two numbers. It follows that $t(\alpha) \in \mathcal{O}(N^n)$. For $s(\alpha)$ we can see that

  $$s(\alpha) = \|[\![\alpha]\!]_{\mathscr{R}}(\mathscr{I}_w)\|$$
  $$\leq \|[\![\beta_1]\!]_{\mathscr{R}}(\mathscr{I}_w)\| + \|[\![\beta_2]\!]_{\mathscr{R}}(\mathscr{I}_w)\| + C$$

  And therefore $s(\alpha) \in \mathcal{O}(N)$.
  — $\alpha = \beta_1 * \beta_2$: The proof works analogously to the proof for the case $\alpha = \beta_1 + \beta_2$.
  — $\alpha = -\beta$: We know that the invariant holds for $\beta$. Further $t(\alpha) = t(\beta) + x$, where $x$ is the time needed for inversion of the result for $\beta$. We know that $x$ is polynomial in $s(\beta)$, which we know to be in $\mathcal{O}(N)$. Therefore $x \in \mathcal{O}(N^l)$, where $l$ is the degree of the polynomial bounding the time needed to invert a number. It follows that $t(\alpha) \in \mathcal{O}(N^n)$. For $s(\alpha)$ we can see that

  $$s(\alpha) = \|[\![\alpha]\!]_{\mathscr{R}}(\mathscr{I}_w)\|$$
  $$\leq \|[\![\beta]\!]_{\mathscr{R}}(\mathscr{I}_w)\| + C$$

  And therefore $s(\alpha) \in \mathcal{O}(N)$.
  — $\alpha = \beta^{-1}$: The proof works analogously to the proof for the case $\alpha = -\beta$.

  □

**Theorem 22** (Ground Complexity). *Let $\Pi$ be a variable-free program s.t. each semiring in $\Pi$ is efficiently encoded. Then*

- *MC is co-NP-complete.*
- *(propositional) SE is co-NP-complete.*

- *SAT is $\sigma_2^p$-complete.*

*Proof (sketch).* The hardness parts are inherited from the complexity of the respective problems for disjunctive logic programs (Dantsin et al. 2001; Lin 2002): The disjunctive logic programming rule

$$a_1 \vee \cdots \vee a_n \leftarrow b_1, \ldots, b_m, \neg c_1, \ldots, \neg c_k$$

is strongly equivalent to the $\mathscr{AC}$-rule

$$1 =_\mathbb{B} a_1 + \ldots + a_n \leftarrow b_1, \ldots, b_m, \neg c_1, \ldots, \neg c_k.$$

The memberships follow from the possibility of applying guess and check algorithms. We only need that given $(\mathscr{I}^H, \mathscr{I}^T)$ and algebraic constraint $k \sim_\mathscr{R} \alpha$, we can decide in polynomial time whether $\mathscr{I}_H \models k \sim_\mathscr{R} \alpha$. This is possible since we know that $\mathscr{R}$ is efficiently encoded: We only need to perform polynomially many additions, multiplications and inversions which each take polynomial time as Theorem II says. □

**Theorem 23** (Non-gound Complexity). *Let $\Pi$ be a safe program such that each semiring in $\Pi$ is efficiently encoded. Then*
(i) *MC is in EXPTIME, and co-NP$^{NP^{PP}}$-hard (thus also co-NP$^{PP}$-hard and NP$^{PP}$-hard).*
(ii) *SAT is undecidable.*
(iii) *SE is undecidable.*

*Proof (sketch).* (i) Given the interpretation $\mathscr{I}$ (as set of ground atoms), we can iterate over all $\mathscr{I}' \subsetneq \mathscr{I}$ and check $(\mathscr{I}, \mathscr{I}, H) \models r'$, as well as $(\mathscr{I}', \mathscr{I}, H) \models r'$ for each ground instance $r'$ of a rule $r \in \Pi$ in exponential time. The iteration and considering one ground instance $r'$ at a time is feasible in polynomial space; the evaluation of algebraic constraints $k \sim_\mathscr{R} \alpha$ is feasible in exponential time, since if $\alpha$ is of the form $\Sigma y_1, \ldots, y_n \alpha'(y_1, \ldots, y_n)$ where $\alpha'$ is quantifier-free, by safety of the program each $y_i$ must occur in some atom $p(\vec{x})$. That is, to evaluate $\alpha$, we only need to consider values $\xi(y_i)$ for $y_i$, $i = 1, \ldots, n$ that occur in the interpretation $\mathscr{I}$. There are exponentially many such $\xi$; for each of them, the value of $\alpha'(\xi(y_1), \ldots, \xi(y_n))$ can be computed in polynomial time given that $\mathscr{R}$ is efficiently encoded, yielding a value $r_\xi$ such that $e(r_\xi)$ occupies polynomially many bits. The aggregation $\Sigma_\xi r_\xi$ over all $\xi$ is then feasible in exponential time by the assertion that $\|r_1 \oplus r_2\| \leq \|r_1\| + \|r_2\| + c$ and that $e(r_1 \oplus r_2)$ is computable in polynomial time given $e(r_1), e(r_2)$.

We note that the value of $\Sigma_\xi r_\xi$ may under the assertions occupy exponentially many bits; under stronger assumptions on the encoding $e(r)$, a smaller upper bound may be derived. E.g., we obtain membership in PSPACE if it is ensured that for the addition, we have the stronger condition $\|r_1 \oplus r_2 \oplus \cdots \oplus r_n\| \leq (1 + \log n) \max_i \|r_i\| + c$, for every $n \geq 2$ where $c$ is a constant. For example, the canonical semiring $\mathbb{N}$ of the natural numbers satisfies this property.

The co-NP$^{NP^{PP}}$-hardness is due to a reduction from AE-MAJSAT, which asks whether for a Boolean formula $\phi(x_1, \ldots, x_n)$ for all assignments to $x_1, \ldots, x_m$ a partial assignment to $x_{m+1}, \ldots, x_k$

exists s.t. more than $2^{n-k-1}$ of the assignments to $x_{k+1},\ldots,x_n$ satisfy $\phi(\vec{x})$. Then the program

$$e(0) \leftarrow$$
$$e(1) \leftarrow$$
$$1 =_{\mathbb{B}} a_1(0) + a_1(1) \leftarrow$$
$$\ldots$$
$$1 =_{\mathbb{B}} a_m(0) + a_m(1) \leftarrow$$
$$1 =_{\mathbb{B}} a_1(0) * a_1(1) \leftarrow a_1(X_1),\ldots,a_m(X_m), e(X_{m+1}),\ldots,e(X_k),$$
$$2^{n-k-1} <_{\mathbb{N}} e(X_{k+1}) * \ldots * e(X_n) * \phi(\vec{X})$$
$$\ldots$$
$$1 =_{\mathbb{B}} a_m(0) * a_m(1) \leftarrow a_1(X_1),\ldots,a_m(X_m), e(X_{m+1}),\ldots,e(X_k),$$
$$2^{n-k-1} <_{\mathbb{N}} e(X_{k+1}) * \ldots * e(X_n) * \phi(\vec{X})$$

has an equilibrium model $\mathscr{I} = \{a_1(0), a_1(1),\ldots,a_m(0), a_m(1), e(0), e(1)\}$ iff the answer for AE-MAJSAT is yes.

Assume the answer for AE-MAJSAT is yes. Then for every subset $\mathscr{I}' \subseteq \mathscr{I}$ we have $(\mathscr{I}', \mathscr{I}, H) \not\models \Pi$. If we remove $e(i)$ or both $a_i(0)$ and $a_i(1)$ for some $i$, this is clear. Otherwise, we know that for each $i$ some $a_i(j_i)$ holds. Then for these values there exist values $j_{m+1},\ldots,j_k$ s.t. $\phi(j_1,\ldots,j_k,X_{k+1},\ldots,X_n)$ is a yes instance for MAJSAT. Therefore the body

$$a_1(j_1),\ldots,a_m(j_m), e(X_{j+1}),\ldots,e(j_k), 2^{n-k-1} <_{\mathbb{N}} e(X_{k+1}) * \ldots * e(X_n) * \phi(\vec{X})$$

is satisfied and if $(\mathscr{I}', \mathscr{I}, H) \models \Pi$ we know that $\mathscr{I}' = \mathscr{I}$.

On the other hand if the answer for AE-MAJSAT is no, due to the partial assignment $j_1,\ldots,j_m$ to the variables $x_1,\ldots,x_m$, for $\mathscr{I}' = \{a_1(j_1),\ldots,a_m(j_m), e(0), e(1)\}$ we have $(\mathscr{I}', \mathscr{I}, H) \models \Pi$, and therefore $\mathscr{I}$ is not an equilibrium model.

(ii) The undecidable Mortal Matrix Problem asks whether any product of matrices in $X = \{X_1,\ldots,X_n\} \subset \mathbb{Z}^{d \times d}$ evaluates to the zero matrix $0_d$ (Cassaigne et al. 2014). $(\mathbb{Z}^{d \times d}, +, \cdot, 0_d, 1_d)$ is efficiently encodable. The program

$$p(X_i) \leftarrow \quad (i = 1,\ldots,n)$$
$$\bot \leftarrow \neg p(0_d)$$
$$p(Y) \leftarrow p(Z_1), p(Z_2), Y =_{\mathbb{Z}^{d \times d}} Z_1 * Z_2$$

has a stable model iff the answer to the mortal matrix problem for $X$ is yes, since $p(0_d)$ needs to be supported.

(iii) Let $\Pi$ be the program from above. Then the answer to the mortal matrix problem for $X$ is yes iff $\Pi$ is strongly equivalent to $\Pi' = \Pi \setminus \{\bot \leftarrow \neg p(0_d)\}$. This can be seen as follows.

As $\Pi'$ has no negation, its HT-models are the interpretations $(\mathscr{I}', \mathscr{I})$ where both $\mathscr{I}'$ and $\mathscr{I}$ are closed under the rules of $\Pi'$, sets $S$ such that $p(X_1),\ldots,p(X_n) \in S$ and whenever $p(Y), p(Z) \in S$ then also $p(Y * Z) \in S$. Similarly, the HT-models of $\Pi$ are the interpretations $(\mathscr{I}', \mathscr{I})$ where $\mathscr{I}'$ and $\mathscr{I}$ are closed under the rules of $\Pi'$ and in addition $p(0_d) \in \mathscr{I}'$.

Therefore, $\Pi \equiv_s$ and $\Pi'$ iff $p(0_d) \in L$, where $L$ is the least set closed under the rules of $\Pi'$, which holds iff the answer for the mortal matrix problem on $X$ is yes. $\square$

**Corollary III.** *When the program $\Pi$ is over the semiring $\mathbb{N}$ of the natural numbers, we have co-NP$^{NP^{PP}}$-completeness for MC.*

*Proof.* The hardness is due to the proof of the previous theorem. The membership follows from the fact that we can check the satisfaction of constraints over $\mathbb{N}$ using a PP oracle.

This can be seen as follows. We can evaluate weighted formulas over $\mathbb{N}$ of the form $\Sigma y_1 \ldots \Sigma y_n \alpha$ where $\alpha$ is quantifier-free using a #P oracle: we can non-deterministically choose an assignment $\xi$ to $y_1, \ldots, y_n$, calculate $r = [\![\alpha(\xi)]\!]_\mathbb{N}(\mathscr{I}_w)$ in polynomial time and generate $r$ accepting branches.

Since $P^{PP}$ is equal to $P^{\#P}$ also co-$NP^{NP^{PP}}$ is equal to co-$NP^{NP^{\#P}}$.

As for the co-$NP^{NP^{\#P}}$ membership: Given a program $\Pi$ and a potential equilibrium model $\mathscr{I}$ we can guess a subset $\mathscr{I}' \subsetneq \mathscr{I}$ and check whether $(\mathscr{I}', \mathscr{I}, H) \models \Pi$. The latter can be achieved in co-$NP^{\#P}$ by guessing a rule $r \in \Pi$ and an assignment $\xi$ to its global variables. Then we can check whether $(\mathscr{I}', \mathscr{I}, H) \models r(\xi)$ in $P^{\#P}$ by checking satisfaction of each atom and constraint in $r(\xi)$. □

We note that MC is decidable for $\mathscr{AC}$-programs over the natural numbers while SAT and SE are undecidable. This may not be much surprising from Theorem **??**, given that the semiring $\mathbb{Z}^{d \times d}$ is (efficiently) encodable to $\mathbb{N}$. The undecidability can directly be shown by a reduction from solving Diophantine equations, i.e., polynomial equations $P(x_1, ..., x_n) = 0$ in variables $x_1, \ldots, x_n$ over the integers, which by Matiyasevich's celebrated result is undecidable; this holds if the solutions are restricted to the natural numbers (Matiyjasevich 1996). We can equivalently consider polynomial equations $P(x_1, ..., x_n) = Q(x_1, ..., x_n)$ where all coefficients in the polynomial expressions $P(x_1, \ldots, x_n)$ and $Q(x_1, \ldots, x_n)$ are non-negative. We then can write a program $\Pi$ consisting of the rules

$$\begin{aligned}
&\mathtt{n}(0) \leftarrow . \\
&\mathtt{n}(X) \leftarrow \mathtt{n}(Y), X =_\mathbb{N} 1 + Y. \\
&\mathtt{sol} \leftarrow \mathtt{n}(X_1), ..., \mathtt{n}(X_n), Y =_\mathbb{N} P(X_1, ..., X_n), Y =_\mathbb{N} Q(X_1, ..., X_n). \\
&\bot \leftarrow \neg \mathtt{sol}.
\end{aligned}$$

The program $\Pi$ is safe and it has a (unique) equilibrium model (in which $\mathtt{sol}$ is true) iff a solution to $P(x_1, ..., x_n) = Q(x_1, ..., x_n)$ exists. Furthermore, the existence of an equilibrium model is equivalent to $\Pi \equiv_s \Pi \setminus \{\bot \leftarrow \neg \mathtt{sol}\}$.

**Definition IV** (Finite Groundability). *Let $\sigma$ some semiring signature and $\Pi$ an $\mathscr{AC}$-program over $\sigma$. Then $\Pi$ is* finitely groundable *if there is a signature $\sigma' = \langle \mathscr{D}, \mathscr{P}, \mathscr{X}, \mathscr{S}, r \rangle$ s.t. the equilibrium models of $\Pi$ over $\sigma$ are the same as the equilibrium models over $\sigma'$ and $\mathscr{D}$ is finite.*

**Theorem V.** *For finitely ground programs over $\sigma' = \langle \mathscr{D}, \mathscr{P}, \mathscr{X}, \mathscr{S}, r \rangle$ ($|\mathscr{D}| < \infty$) that only use computable semirings, SAT and SE are decidable.*

*Proof.* We can replace universally quantified formulas with finite conjunctions over all the substitutions and existentially quantifies formulas with finite disjunctions over all the substitutions.

For variable free programs we have decidability when all the semirings are computable. □

**Theorem VI.** *Let $\Pi$ be a safe program over $\sigma$ without value invention, where all algebraic constraints in heads are domain restricted. Then $\Pi$ is finitely ground over $\sigma' = \langle \mathscr{D}, \mathscr{P}, \mathscr{X}, \mathscr{S}, r \rangle$, where $\mathscr{D}$ is the subset of domain values that occur in $\Pi$.*

*Proof.* Let $\mathscr{I}$ be a $\sigma$-interpretation s.t. $(\mathscr{I}, \mathscr{I}, T) \models \Pi$. Then for $\mathscr{I}'$ obtained from $\mathscr{I}$ by removing all atoms that contain constants not from $\mathscr{D}$, we have $(\mathscr{I}', \mathscr{I}, H) \models \Pi$. This can be seen as follows: Assume $r \in \Pi$ with global variables $x_1, \ldots, x_n$. Since $r$ is safe and does not contain value

invention, the body of *r* can only be satisfied for substitutions of the variables with elements from $\mathscr{D}$. Therefore, if the head of *r* is an atom $p(\vec{x})$, we can only derive $p(\vec{\xi})$ for substitutions from $\mathscr{D}$ and therefore, if the rule was satisfied previously, it is still satisfied.

Otherwise, if the head of *r* is a constraint, we know that it is domain restricted, i.e. of the form

$$k \sim_{\mathscr{R}} \neg\neg\alpha(\vec{X}) * (\alpha(\vec{X}) \to \beta(\vec{X})) * \gamma(\vec{X}),$$

where $\alpha(\vec{X}), \beta(\vec{X})$ are syntactically domain independent and all atoms in $\gamma(\vec{X})$ are locally ground.

Let $\vec{\xi}$ be some assignments to $\vec{X}$ over the original domain.

We consider first $\neg\neg\alpha(\vec{\xi})$. It holds that

$$\begin{aligned}
[\![\neg\neg\alpha(\vec{\xi})]\!]_{\mathscr{R}}(\mathscr{I}_H) = e_\oplus &\iff [\![\neg\alpha(\vec{\xi})]\!]_{\mathscr{R}}(\mathscr{I}_H) = e_\otimes \vee [\![\neg\alpha(\vec{\xi})]\!]_{\mathscr{R}}(\mathscr{I}_T) = e_\otimes \\
&\iff [\![\alpha(\vec{\xi})]\!]_{\mathscr{R}}(\mathscr{I}_H) = e_\oplus \wedge [\![\alpha(\vec{\xi})]\!]_{\mathscr{R}}(\mathscr{I}_T) = e_\oplus \vee [\![\alpha(\vec{\xi})]\!]_{\mathscr{R}}(\mathscr{I}_T) = e_\oplus \\
&\iff [\![\alpha(\vec{\xi})]\!]_{\mathscr{R}}(\mathscr{I}_T) = e_\oplus
\end{aligned}$$

Therefore this part of the formula is not influenced by $\mathscr{I}'$.

Secondly we consider $\alpha(\vec{\xi}) \to \beta(\vec{\xi})$. We only need to consider this value if $[\![\neg\neg\alpha(\vec{\xi})]\!]_{\mathscr{R}}(\mathscr{I}_H)$ is unequal to zero, i.e. if $[\![\alpha(\vec{\xi})]\!]_{\mathscr{R}}(\mathscr{I}_T)$ is unequal to zero. Now if $[\![\alpha(\vec{\xi}) \to \beta(\vec{\xi})]\!]_{\mathscr{R}}(\mathscr{I}_T)$ is unequal to zero this implies that $[\![\beta(\vec{\xi})]\!]_{\mathscr{R}}(\mathscr{I}_T)$ is unequal to zero. If all the values in $\vec{\xi}$ are from $\mathscr{D}$, there is no change. Otherwise we know that $[\![\alpha(\vec{\xi})]\!]_{\mathscr{R}}(\mathscr{I}_H) = [\![\beta(\vec{\xi})]\!]_{\mathscr{R}}(\mathscr{I}_H) = e_\oplus$ since $\alpha(\vec{X})$ and $\beta(\vec{X})$ are syntactically domain independent and therefore have value $e_\oplus$ for values that are not mentioned in the interpretation $\mathscr{I}'$ (see proof of the invariance of the support for syntactically domain independent weighted formulas). It follows that $[\![\alpha(\vec{\xi}) \to \beta(\vec{\xi})]\!]_{\mathscr{R}}(\mathscr{I}_T)$ is also unequal to zero.

Since $\gamma(\vec{X})$ contains only locally ground atoms, the restriction of the interpretation to $\mathscr{I}'$ does not change the value of $\gamma(\vec{\xi})$.

Therefore

$$[\![\neg\neg\alpha(\vec{X}) * (\alpha(\vec{X}) \to \beta(\vec{X})) * \gamma(\vec{X})]\!]_{\mathscr{R}}(\mathscr{I}_H) = [\![\neg\neg\alpha(\vec{X}) * (\alpha(\vec{X}) \to \beta(\vec{X})) * \gamma(\vec{X})]\!]_{\mathscr{R}}(\mathscr{I}_T)$$

and

$$\begin{aligned}
\mathscr{I}_H &\models k \sim_{\mathscr{R}} \neg\neg\alpha(\vec{X}) * (\alpha(\vec{X}) \to \beta(\vec{X})) * \gamma(\vec{X}) \\
\iff \mathscr{I}_T &\models k \sim_{\mathscr{R}} \neg\neg\alpha(\vec{X}) * (\alpha(\vec{X}) \to \beta(\vec{X})) * \gamma(\vec{X}).
\end{aligned}$$

We see that since $(\mathscr{I}, \mathscr{I}, T) \models \Pi$ also $(\mathscr{I}', \mathscr{I}, T) \models \Pi$. Therefore $\mathscr{I}$ can only be an equilibrium model if it contains only constants from $\mathscr{D}$, which implies that $\Pi$ is finitely ground over $\sigma'$. $\square$

**Theorem 25.** *For safe programs without value invention where all algebraic constraints in rule heads are domain restricted and all semirings are computable, both SAT and SE are decidable.*

*Proof.* This result follows easily from Theorems V and VI. $\square$

## References

CABALAR, P., FANDINNO, J., SCHAUB, T., AND WANKO, P. 2020. An ASP semantics for constraints involving conditional aggregates. *arXiv preprint arXiv:2002.06911*.

CASSAIGNE, J., HALAVA, V., HARJU, T., AND NICOLAS, F. 2014. Tighter undecidability bounds for matrix mortality, zero-in-the-corner problems, and more. *CoRR abs/1404.0644*.

DANTSIN, E., EITER, T., GOTTLOB, G., AND VORONKOV, A. 2001. Complexity and expressive power of logic programming. *ACM Comput. Surv. 33,* 3, 374–425.

DEMOLOMBE, R. 1992. Syntactical characterization of a subset of domain-independent formulas. *Journal of the ACM (JACM) 39,* 1, 71–94.

GREEN, T. J., KARVOUNARAKIS, G., AND TANNEN, V. 2007. Provenance semirings. In *Proc. ACM PODS'07*. ACM, 31–40.

LIFSCHITZ, V., PEARCE, D., AND VALVERDE, A. 2001. Strongly equivalent logic programs. *ACM TOCL 2,* 4, 526–541.

LIFSCHITZ, V., PEARCE, D., AND VALVERDE, A. 2007. A characterization of strong equivalence for logic programs with variables. In *International Conference on Logic Programming and Nonmonotonic Reasoning*. Springer, 188–200.

LIN, F. 2002. Reducing strong equivalence of logic programs to entailment in classical propositional logic. *KR 2*, 170–176.

MATIYJASEVICH, Y. 1996. Hilbert's tenth problem: what can we do with Diophantine equations? English version of a talk given by the author. Available at http://logic.pdmi.ras.ru/~yumat/personaljournal/H10history/H10histe.pdf.

PEARCE, D. AND VALVERDE, A. 2006. Quantified Equilibrium Logic and the First Order Logic of Here-and-There. *Technical Report MA-06-02*.


\end{document}